\documentclass[11pt, a4paper]{article}                     
\usepackage{amsmath, amsthm,amssymb, anysize, graphicx, enumerate, subfigure, color}
\usepackage{color}
\usepackage{graphicx}
\usepackage{algorithmic}
\usepackage{algorithm}
\usepackage{graphicx}
\usepackage{multicol}
\usepackage{multirow}
\usepackage{url}
\newtheorem{lemma}{Lemma}

\newtheorem{theorem}{Theorem}

\newtheorem*{proposition}{Proposition}
\newtheorem{definition}{Definition}

\newtheorem{example}{Example}
\newtheorem{remark}{Remark}

\newcommand{\new}{\newcommand}
\new{\bg}{\begin}
\new{\lp}{\left(}
\new{\rp}{\right)}
\new{\lb}{\left\{}
\new{\rb}{\right\}}
\new{\lsq}{\left[}
\new{\rsq}{\right]}

\new{\B}{B}
\new{\Bn}{{\B_\n}}
\new{\Bnd}{{\B^\dd_\n}}
\new{\C}{\mathcal C}
\new{\Cn}{{\C_\n}}
\new{\grad}{\nabla}
\new{\gradn}{\hat{\grad}}
\new{\der}{D}
\new{\derh}{\partial}
\new{\dern}{\hat{\der}}

\new{\fdag}{f^{\dag}_\rho}

\new{\reg}{\Omega}
\new{\Regze}{\reg^{ \textrm{D}}_0}
\new{\Reg}{\reg^{ \textrm{D}}_1}
\new{\Regn}{\widehat{\Omega}^D_1}
\new{\RR}{\omega}

\new{\jj}{a}
\new{\jjj}{b}
\new{\jjjj}{c}
\new{\ii}{i}
\new{\iii}{j}
\new{\itext}{t}
\new{\itint}{q}

\new{\step}{\sigma}
\new{\pone}{\tau}
\new{\ptwo}{\tau}
\new{\taumu}{\nu}

\new{\fff}{\end{enumerate}}
\new{\iiii}{\begin{itemize}}
\new{\ffff}{\end{itemize}}
\new{\mfi}{\begin{eqnarray*}}
\new{\mff}{\end{eqnarray*}}
\new{\mfni}{\begin{eqnarray}}
\new{\mfnf}{\end{eqnarray}}
\new{\beeq}[2]{\begin{equation}\label{#1}{#2}\end{equation}}
\new{\eqn}[1]{(\ref{#1})}

\new{\room}{\ \ \ \ }
\new{\card}{\#}

\new{\zz}[1]{z_{j,{#1}}}
\new{\nor}[1]{\|{#1}\|}
\new{\norh}[1]{\|{#1}\|_\hh}
\new{\scal}[2]{\langle{#1},{#2}\rangle}
\new{\scalh}[2]{\langle{#1},{#2}\rangle_\hh}
\new{\set}[1]{\{{#1}\}}
\new{\com}{{\mathbb C}}
\new{\rone}{{\mathbb R}}
\new{\nat}{{\mathbb N}}

\new{\fz}{\hat{f}^{\pone}}
\new{\fzn}{\hat{f}^{\pone_\n}}

\new{\argmin}{\operatornamewithlimits{argmin}}
\new{\argmax}{\operatornamewithlimits{argmax}}
\new{\Prob}[1]{\mathrm{P}\!\left(\, #1 \right)}
\new{\E}{{\mathbb E}}
\new{\eps}{\varepsilon}

\new{\marg}{\rho_\X}
\new{\prob}{\rho}
\new{\dd}{d}
\new{\deff}{|\I|}
\new{\tr}{\vz}
\new{\n}{n}
\new{\vx}{{\mathbf x}}
\new{\vy}{{\mathbf y}}
\new{\vz}{{{\mathbf z}_\n}}
\new{\kk}{k}
\new{\kkx}{K_x}
\new{\ka}{\kappa_1}
\new{\kader}{\kappa_2}
\new{\da}{{\delta_2}}
\new{\db}{{\delta_3}}
\new{\dder}{{\delta_1}}
\new{\pr}{\eta}

\new{\matD}{{\mathrm D}}
\new{\matL}{{\mathrm L}}
\new{\matK}{{\mathrm K}}
\new{\matZ}{{\mathrm Z}}
\new{\err}{{\mathcal E}}
\new{\errm}{{\mathcal E}_{n,\tau}}
\new{\emp}{\widehat{{\mathcal E}}}
\new{\hh}{{\mathcal H}}
\new{\ldue}{L^2(\X,\marg)}
\new{\lduen}{L^2(\X,\rho_\vx)}
\new{\la}{\lambda}

\new{\Vh}{V}
\new{\Vhn}{{\hat{V}}}

\new{\Ik}{I_\kk}
\new{\IIk}{I^*_\kk}

\new{\Tx}{T_n}
\new{\Sx}{\hat{S}}
\new{\Sxa}{\hat{S}^\act}
\new{\T}{T}
\new{\Kx}{{K_{\mathbf x}}}
\new{\K}{K}
\new{\pen}{\text{\sf pen}}

\new{\vre}{\vf_\rho}
\new{\re}{f_\rho}
\new{\X}{\mathcal X}
\new{\Z}{\mathcal Z}
\new{\Y}{\mathcal Y}

\new{\I}{{R_\rho}}
\new{\Izn}{{\hat{R}^{\pone_\n}}}
\new{\Iznon}{{\hat{R}^\pone}}
\new{\act}{A}

\new{\EL}{\err^{\theta}}

\usepackage{color, palatino}

\begin{document}

\title{Nonparametric Sparsity and Regularization}

\author{Lorenzo Rosasco$^1$ \and Silvia Villa$^2$ \and Sofia Mosci $^3$ \and Matteo Santoro $^4$ \and Alessandro Verri $^5$}

\maketitle

\begin{center}
$^1$ CBCL, McGovern Institute, 
       Massachussets Institute of Technology, USA\\ and Istituto Italiano di Tecnologia, ITALY, lrosassco@mit.edu\\
$^2$ Istituto Italiano di Tecnologia, ITALY, silvia.villa@iit.it\\
$^3$ DIBRIS,         University of Genova,  ITALY, sofia.mosci@unige.it\\
$^4$ Istituto Italiano di Tecnologia, ITALY, matteo.santoro@iit.it\\
$^5$ DIBRIS,         University of Genova,  ITALY, alessandro.verri@unige.it
\end{center}

\maketitle

\begin{abstract}
 In this work we are interested in the problems of  supervised learning and variable selection 
 when the input-output dependence  is described by a  nonlinear function depending on a  few variables.  
 Our goal is to consider a sparse nonparametric model, hence  avoiding linear or additive models. 
The key idea is to measure the importance of each variable in the model by making use of partial derivatives.
 Based on this intuition we propose a new notion of nonparametric sparsity and a corresponding  
 least squares regularization scheme. Using concepts and results  from the theory of reproducing 
 kernel Hilbert spaces and proximal methods, we show that the proposed learning algorithm corresponds
 to a minimization problem which can be provably solved by an iterative procedure. 
 The consistency properties of the obtained estimator are studied both in terms of prediction and 
 selection performance. An extensive empirical analysis shows that the proposed method
 performs favorably with respect to the state-of-the-art methods.
\end{abstract}

\section{Introduction}
 
It is now common to see practical applications, for example in bioinformatics and computer vision,  where the dimensionality of the data  
is in the order of hundreds, thousands  and even tens of thousands.
It is known that learning  in such a  high dimensional regime is feasible only if the quantity to be estimated 
satisfies some regularity assumptions \cite{degylu96}.
In particular, the idea behind, so called,  {\em sparsity} is that the quantity  
of interest depends only on a  few relevant variables (dimensions). In turn,  this latter assumption is often at the basis
of the construction of interpretable data models, since the relevant dimensions 
allow for a compact, hence interpretable,  representation. An instance of the above situation is the problem of 
learning from samples a multivariate function  which depends only on a (possibly small) subset of {\em relevant} variables.
Detecting such  variables is the problem of variable selection. 

Largely motivated by recent advances in compressed sensing \cite{carota06,do06},   the above problem has been extensively studied 
under the assumption that the function of interest (target function) depends {\em linearly} 
to the relevant variables. 
While a naive approach (trying all possible subsets of variables) would not be computationally 
feasible it is known that meaningful approximations can be found either by greedy methods \cite{Tropp07},  or 
convex relaxation  ($\ell^1$ regularization a.k.a. basis pursuit or LASSO \cite{Tibshirani96, chen1999, efron04}). 
In this context efficient algorithms (see \cite{schmidt2007,loris09} and references therein) 
as well as  theoretical guarantees are now available
(see  \cite{buhvan11} and references therein).
In this paper we are interested into the situation where the target function  depends {\em non-linearly} 
to the relevant variables.  This latter situation is much less understood. 
Approaches in the literature are mostly restricted to additive models  \cite{hastie1990}. In such models the target function is assumed to be a sum of (non-linear) 
univariate functions. Solutions to the problem of variable selection in this class of models 
include \cite{ravikumar2008} and are related to  multiple kernel learning \cite{bach04}.  
Higher order additive models can be further considered,   
encoding explicitly dependence among the variables -- for example assuming  the target function to be  also  sum of   
functions  depending on  couples, triplets etc. of variables, as in \cite{lin2006} and  \cite{bach09}. Though this approach 
provides a more interesting, while still interpretable, model,  its size/complexity is essentially  more than exponential in the initial variables.
Only a few works, that we discuss in details in Section \ref{Sec:relwor},  have considered notions of sparsity beyond additive models.

In this paper, we  propose a new approach  based  on the idea that the importance of a variable,  while learning 
a non-linear functional relation, can be captured by the corresponding partial derivative. This observation suggests a way to define a new notion of nonparametric sparsity and a corresponding  regularizer which favors functions where most partial derivatives are  essentially zero.  The question is how to make this intuition  precise and how to derive a  feasible computational learning scheme.
The first observation is  that, while we cannot measure a partial  derivative {\em everywhere},  we can do it at the training set points and hence design 
a data-dependent regularizer.  In order to derive an actual algorithm we have to consider two further issues: How can we estimate reliably
partial derivatives in high dimensions? How can we ensure that the data-driven penalty is sufficiently stable?
The theory of reproducing kernel Hilbert spaces (RKHSs) provides us with tools to answer both questions.
In fact, partial derivatives in a RKHS  are bounded linear functionals and hence have a suitable representation that 
allows  efficient computations. Moreover, the norm in the RKHS provides a natural further regularizer
ensuring stable behavior of the empirical, derivative based  penalty.
 Our contribution is threefold.
 First, we propose a new notion of sparsity and discuss a corresponding regularization scheme using concept from the theory of reproducing kernel Hilbert spaces.
 Second, since the proposed algorithm corresponds to the minimization of a convex, but not differentiable functional, we 
develop a suitable optimization procedure relying on forward-backward splitting and proximal methods. 
Third, we study properties of the proposed methods both in theory, in terms of statistical consistency, and in practice, by means of an extensive set of experiments.

Some preliminary  results  have appeared in a short conference version of this paper \cite{rosasco2010nvs}.
With respect to the conferecen  version, the current version contains:  the detailed discussion of the derivation 
of the algorithm with all the proofs, the consistency results of Section 4, an augmented set of experiments and several  further discussions.
The paper  is organized as follows. In section \ref{sec:approach} we discuss our approach and present the main results in the paper.
In Section \ref{sec:computation} we discuss the computational aspects of the method. In Section \ref{sec:cons} we prove consistency results. In Section \ref{sec:emp_val}
we provide an  extensive empirical analysis. Finally in Section \ref{sec:disc} we conclude with a summary of our study and a discussion of future work.

\section{Problem Setting and Previous Work}\label{Sec:relwor}
Given a training set  $\vz =(\vx,\vy)=(x_\ii, y_\ii)_{i=1}^n$ of input output pairs, 
with $x_\ii\in \X \subseteq \rone^\dd$ and $y_\ii \in\Y \subseteq \rone$, 
we are interested  into learning about the functional relationship between input and output. 
More precisely, in statistical learning the data are assumed  to be sampled identically and independently from a probability 
measure $\rho$ on $\X\times\Y$ so that if we measure the error by the square loss function, the regression 
function $f_\rho (x)=\int yd\rho(x,y)$ minimizes the expected risk $\err(f)=\int (y-f(x))^2d\rho(x,y)$.\\
Finding an estimator $\hat{f}$ of $f_\rho$ from finite data is possible, if $f_\rho$ satisfies some suitable prior assumption \cite{degylu96}. 
In this paper we are interested in the case where the regression function is {\em sparse} in the sense that it depends 
only  on a subset $R_\rho$ of  the possible $\dd$ variables. Estimating the set $R_\rho$ of  {\em relevant} variables is the problem of variable selection.

\paragraph{Linear and additive models}

The  sparsity requirement  can be made precise considering linear functions  $f(x)=\sum_{\jj=1}^\dd\beta_\jj x^\jj$ with $x=(x^1, \dots, x^\dd)$. In this case the 
sparsity of a function is  quantified by the so called {\em zero-norm} $\Omega_0(f)=\#\{\jj=1, \dots, \dd~|~\beta_\jj\neq 0\}$.
The zero norm,  while natural for variable selection,  does not lead to efficient algorithms and 
is often replaced by the $\ell^1$ norm, that is $\Omega_1(f)=\sum_{\jj=1}^{\dd}|\beta_\jj|$.  
This approach has been studied extensively and is now fairly  well understood, see \cite{buhvan11} and references therein. Regularization
with $\ell^1$ regularizers, obtained by minimizing 
$$
\emp(f)+\la \Omega_1(f),\quad \quad \emp(f)= \frac{1}{\n}\sum_{\ii=1}^\n(y_\ii - f(x_\ii))^2, 
$$
can be solved efficiently and,  under suitable conditions, provides a solution 
close to that of the zero-norm regularization. 

The above scenario can be generalized 
to additive models  $f(x)=\sum_{\jj=1}^\dd f_\jj(x^\jj)$, where $f_\jj$ are univariate functions in some (reproducing kernel) 
Hilbert spaces $\hh_\jj$, $\jj=1, \dots, \dd$. In this case the analogous 
of the zero-norm and the $\ell^1$ norm are  $\Omega_0(f)=\#\{\jj\in\{1,\dots,\dd\}~:~\nor{f_\jj}\neq 0\}$ 
and $\Omega_1(f)=\sum_{\jj=1}^{\dd}\nor{f_\jj}$, respectively.
This latter setting, related to multiple kernel learning \cite{bach04,bach08},  has been considered for example 
in \cite{ravikumar2008},  see also \cite{kolyua10} and references therein.
Considering additive models limits the way in which the variables can interact.
This can be partially alleviated considering higher order terms in 
the model as it is done  in ANOVA decomposition \cite{wahba1995,gu2002}.
More precisely, we can add to the simplest additive model functions of couples $f_{\jj,\jjj}(x^\jj,x^\jjj)$, triplets $f_{\jj,\jjj, \jjjj}(x^\jj,x^\jjj, x^\jjjj)$, etc. of variables -- see \cite{lin2006}. 
For example one can consider functions of the form 
$f(x)=\sum_{\jj=1}^\dd f_\jj(x^\jj)+\sum_{\jj<\jjj} f_{\jj,\jjj}(x^\jj,x^\jjj)$. In this case the analogous to the
zero and $\ell^1$ norms are  $\Omega_0(f)=\#\{\jj=1, \dots, \dd~:~\nor{f_\jj}\neq 0\} +
\#\{(\jj,\jjj) ~:~\jj<\jjj, ~ \nor{f_{\jjj,\jjjj}}\neq 0\}$ and $\Omega_1(f)=\sum_{\jj=1}^{\dd}\nor{f_\jj}+\sum_{\jj<\jjj}\nor{f_{\jj,\jjj}}$, respectively.
Note that in this case sparsity will not be in general with respect to the original variables but rather with respect 
to the elements in the additive model.
Clearly,  while this approach  provides a more interesting and yet interpretable model, its size/complexity is essentially more than exponential in the number of  variables. 
Some proposed attempts to tackle this problem are based on  restricting the set of  allowed sparsity patterns and  can be found in \cite{bach09}. 
 
\subsection{Nonparametric approaches}
 The above discussion naturally raises the  question:\\
 {\em What if we are interested into learning and performing variable selection when the functions of interest are not described by an additive model?}

Few papers have considered this question. Here we discuss  in some more 
details \cite{lafferty08,bertin08,MilHal10}, \cite{comminges2011},  to which we also refer for further references.\\
The first three papers \cite{lafferty08,bertin08,MilHal10} follow similar approaches focusing on the  point-wise estimation 
of the regression function and of the relevant variables. The basic idea is to start from a locally linear (or polynomial) point wise estimator $f_n(x)$
at a point $x$  obtained from the minimizer of 
\begin{equation}\label{local risk}
\frac 1 n \sum_{i=1}^n (y_i- \scal{w}{x_i-x}_{\rone^d})^2K_H(x_i-x)
\end{equation}
where $K_H$ is a localizing window function depending on a matrix (or a vector) $H$ of smoothing parameters.
Different techniques are used to (locally) select variables.
In the RODEO algorithm \cite{lafferty08}, the localizing window function depends on 
one smoothing parameter per variable and the partial derivative of the local estimator with respect to the 
smoothing parameter is used to select variables. In \cite{bertin08}, selection is considering a {\em local lasso}, that is 
 an $\ell_1$ to  the local empirical risk functional~\eqref{local risk}. In the LABAVS algorithm discussed in 
\cite{MilHal10} several variable selection criterion are discussed including the local lasso, hard thresholding, 
and backward step wise approach. 
The above approaches typically leads to cumbersome computations and 
do not scale well with the dimensionality of the space and with the number of relevant variables.\\
Indeed, in all the above works the emphasis is in the theoretical analysis quantifying  the estimation errors of the proposed methods.
It is shown in  \cite{lafferty08} that the RODEO algorithm is a nearly optimal pointwise estimator of the regression function, 
under assumption on the marginal distribution and the regression functions.
These results are further improved in \cite{bertin08} where optimal rates are derived under milder assumptions 
and sparsistency (the recovery of $R_\rho$) is also studied. Uniform error estimates are derived in \cite{MilHal10} 
(see Section 2.6 in \cite{MilHal10} for further discussions and comparison).
More recently, an estimator  based on the comparison of some well chosen empirical Fourier 
coefficients to a prescribed significance level is described and  studied in \cite{comminges2011}
where a careful statistical analysis is proposed  considering different regimes for $n, d$ and $d^*$, where
$d^*$ is the cardinality of $R_\rho$. Finally, in a slightly different context,  \cite{depewo11} studies the related problem of determining the number of function values 
at adaptively chosen points that are needed in order to correctly estimate the set of globally relevant variables.

\section{Sparsity Beyond linear Models}\label{sec:approach}
In this section we present our approach and summarize our main contributions. 

 \subsection{Sparsity and Regularization using  Partial Derivatives}\label{sec:main_explain}

Our study starts from the observation that, if a  function $f$ is differentiable, the relative importance of
a variable at a point $x$ can be captured by the magnitude of the  corresponding partial 
derivative\footnote{In order for the partial derivatives to be defined at all points we always assume 
that the closure of $\X$ coincides with the closure of its interior.}
$$\left|\frac{\partial f}{\partial x^\jj}\right|.$$
This observation can be developed into a new notion of sparsity and corresponding regularization scheme that we study in the rest of the paper.
We note, that tegularization using derivatives is not new. Indeed, the classical splines (Sobolev spaces) regularization \cite{wahba1990}, as well as more modern techniques such as manifold 
regularization \cite{BelNiy08} use derivatives to measure the regularity of a function. Similarly total variation regularization utilizes derivatives 
to define regular function.
None of the above  methods though allows to capture a  notion of sparsity suitable both for learning and variable selection-- see Remark \ref{rem:der}.

Using partial derivatives to define a new notion of a sparsity and design a regularizer for learning and variable selection requires considering the following two  issues.
First,  we need to quantify  the relevance of a variable beyond a single input point to define a proper (global) notion of sparsity.
If the partial derivative is  continuous
\footnote{In the following, see Remark \ref{smoothAss},  we will see that further appropriate regularity properties on $f$ are needed 
depending on whether the support of $\marg$ is connected or not.} 
then a natural idea is to consider 
\beeq{cont-penalty}{
\left\|\frac{\partial f}{\partial x^\jj}\right\|_{\marg}=
 \sqrt{ \int_{\X}\left(\frac{\partial f(x)}{\partial x^\jj}\right)^2d\marg(x)}.
 } 
where $\marg$ is the marginal probability measure of $\rho$ on $X$. While considering other $L^p$ norms 
is possible,  in this paper we restrict our attention to $L^2$.
A notion of nonparametric  sparsity for a smooth, non-linear function $f$  is captured by the following functional
\beeq{nonparzero}{
\Regze(f)=\#\left \{\jj=1, \dots, \dd~:~\left\|\frac{\partial f}{\partial x^\jj}\right\|_{\marg} \neq 0\right \},
}
and the corresponding  relaxation is 
 $$
\Reg(f)= \sum_{\jj=1}^{\dd} \left\|\frac{\partial f}{\partial x^\jj}\right\|_{\marg}.
$$
The above functionals encode the notion of sparsity that we are going to consider. 
While for linear models, the above definition subsumes the classic notion of sparsity, the above definition 
is non constrained to any (parametric) additive model.

Second, since   $\marg$ is only known through the training set, 
to obtain a practical algorithm  we  start by replacing the $L^2$ norm with an empirical version
$$
\left\|\frac{\partial f}{\partial x^\jj}\right\|_{n}=\sqrt{\frac{1}{\n}\sum_{i=1}^\n \left(\frac{\partial f (x_\ii)}{\partial x^\jj} \right)^2 }
$$
and  by replacing \eqref{cont-penalty} by the data-driven regularizer,
\beeq{penalty}{
\Regn(f)= \sum_{\jj=1}^{\dd} \left\|\frac{\partial f}{\partial x^\jj}\right\|_{\n}.
}
While the above quantity is a natural estimate of \eqref{cont-penalty} in practice it might not be sufficiently stable 
to ensure good function estimates where data are poorly sampled. In the same spirit of manifold 
regularization \cite{BelNiy08},  we then propose to further consider functions in  a reproducing 
kernel Hilbert space (RKHS) defined by a differentiable  kernel and use the penalty, 
$$
\Regn(f)+\taumu\norh{f}^2,
$$
where $\nu$  is a small positive number. The latter terms ensures stability while making the regularizer strongly convex.
This latter property is a key for well-posedeness and generalization, as we discuss in Section \ref{sec:cons}.  
As we will see in the following, RKHS will also be a key tool allowing computations of partial derivative 
of potentially high dimensional functions.

The final  learning algorithm is given by the minimization of the functional 
\begin{equation}\label{algo_init}
 \frac{1}{\n}\sum_{i=1}^\n(y_\ii - f(x_\ii))^2+\pone \left( \sum_{\jj=1}^{\dd} \left\|\frac{\partial f}{\partial x^\jj}\right\|_{n}+\taumu\norh{f}^2\right).
\end{equation}
The remainder of the paper is devoted to the  analysis of  the above regularization algorithm.
Before summarizing our main results we add two  remarks.

\begin{figure}[t]
\begin{center}
\caption{\label{fig:norms} 
Difference between  $\ell^1/\ell^1$ and $\ell^1/\ell^2$  
norm for binary matrices (white = 1, black=0), 
where in the latter case the $\ell^1$ norm is taken over the rows (variables) and the $\ell^2$ norm over the columns (samples). 
The two matrices have the same number of  nonzero entries, 
and thus the same $\ell^1/\ell^1$ norm,  but the value of the $\ell^1/\ell^2$ norm is smaller for the matrix on the right, 
where the nonzero entries are positioned to fill a subset of the rows.
The situation on the right is thus favored by $\ell^1/\ell^2$ regularization. }
\includegraphics[width = \linewidth]{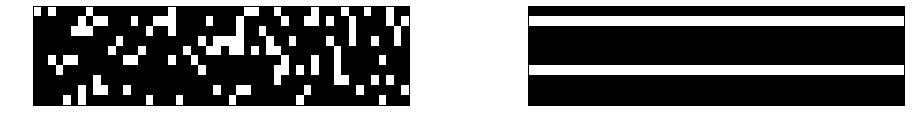}
\end{center}
\end{figure}
\begin{remark}[Comparison with Derivative Based Regulrizers]\label{rem:der}
It is perhaps useful to remark the difference between the regularizer we propose and other derivative based regularizers.
We start by considering 
$$
 \sum_{\jj=1}^{\dd} \left\|\frac{\partial f}{\partial x^\jj}\right\|_{n}^2= \frac{1}{\n}\sum_{\ii=1}^\n  \sum_{\jj=1}^{\dd}
 \left(\frac{\partial f (x_\ii)}{\partial x^\jj} \right)^2=  \frac{1}{\n}\sum_{\ii=1}^\n \nor{\nabla f(x_\ii)}^2, 
$$
where $\nabla f(x)$ is the gradient of $f$ at $x$.
This is essentially a data-dependent version of the classical penalty in Sobolev spaces which writes $\int \nor{\nabla f(x)}^2dx$, where the uniform (Lebesgue) measure is considered. It is well known that while this regularizer measure the smoothness it does not yield any sparsity property. 
 A different derivative based regularizer is given by 
 $
 \frac{1}{\n}\sum_{i=1}^\n  \sum_{\jj=1}^{\dd}
 \left|\frac{\partial f (x_\ii)}{\partial x^\jj} \right|.
 $
Though this  penalty (which we call $\ell^1/\ell^1$) favors sparsity, it only forces partial derivative at points to be zero.
In comparison the regularizer we propose is of the $\ell^1/\ell^2$ type and utilizes 
the square root to ``group'' the values of each partial derivative at different points hence favoring functions for which 
each partial derivative is small at most points. The difference between penalties is illustrated in Figure \ref{fig:norms}.
Finally note that we can also consider  
$\frac{1}{\n}\sum_{\ii=1}^\n \nor{\nabla f(x_\ii)}. $
This regularizer, which is akin to the total variation regularizer $\int \nor{\nabla f(x)}dx$, groups the partial derivatives  differently 
and favors functions with localized singularities rather than  selecting variables.
\end{remark}
\begin{remark}\label{smoothAss} As it is clear from the previous discussion, we quantify the importance of a variable based on the norm of the corresponding partial derivative. This approach makes sense only if
\beeq{teofondcal}{\nor{\frac{\partial f}{\partial x_\jj}}_{\marg}=0\ \Rightarrow \ f \text{ is constant with respect to $x_\jj$.}}
The previous fact  holds trivially if we assume the function $f$ to be continuously differentiable (so that the derivative is pointwise defined, and is a continuous function) and  $\text{supp}\marg$ to be connected. If the latter assumption is not satisfied the situation is more complicated, as the following example shows. Suppose that $\marg$ is the uniform distribution on the disjoint intervals $[-2,-1]$ and $[1,2]$, and $\Y=\{-1,1\}$. Moreover assume that $\rho(y|x)=\delta_{-1},$ if $x\in[-2,-1]$ and $\rho(y|x)=\delta_{1},$ if $x\in[1,2]$. Then, if we consider the regression function 
\[
f(x)=\begin{cases} -1 & \text{if } x\in[-2,-1]\\
1 & \text{if } x\in[1,2]
\end{cases} 
\]
we get that $f'(x)=0$ on the support of $\marg$, although the variable $x$ is relevant. To avoid such pathological situations when $\text{supp}\marg$ is not connected in 
$\mathbb{R}^\dd$ we need to impose more stringent regularity assumptions that basically imply that a function which is constant on a open interval is constant everywhere. This is verified when $f$ belongs to the RKHS defined by a polynomial kernel, or, more generally, an  analytic kernel such as the Gaussian kernel.
\end{remark}

\subsection{Main Results} 

We summarize our main contributions.
\begin{enumerate}
\item Our main contribution is the analysis of the minimization of \eqref{algo_init} and the derivation of a provably convergent iterative  optimization 
procedure. We begin by extending the  representer theorem  \cite{wahba1990} and show that the minimizer of \eqref{algo_init} 
has the  finite dimensional representation 
$$
\fz(x)= 
  \sum_{\ii=1}^n \frac 1 \n \alpha_\ii k(x_\ii,x)+
  \sum_{\ii=1}^n \sum_{\jj=1}^\dd  \frac 1 \n \beta_{\jj\ii} 
  \left.\frac{\partial k(s,x)}{\partial s^\jj}\right|_{s=x_\ii},
$$
  with $\alpha, (\beta_{\jj\ii})_{\ii=1}^\n\in \rone^\n$ for all $\jj=1, \dots, \dd$.
Then, we show that the coefficients in the expansion  can be computed using forwards-backward splitting and proximal methods
\cite{combettes,beck09}. 
More precisely, we present a fast forward-backward splitting algorithm, in which 
 the proximity operator does not admit a closed form and  is thus computed in an approximated way.
Using recent results for proximal methods with approximate proximity operators, we are able to prove 
convergence (and convergence rates) for the overall procedure.
The resulting algorithm  requires  only  matrix multiplications and thresholding operations 
and is in terms of the coefficients $\alpha$ and $\beta$
and  matrices given by the kernel and its first and second derivatives
evaluated at the training set points.\\

\item We study the consistency properties of the obtained estimator. 
We prove that, if the kernel we use is universal, then there exists a choice of $\pone = \pone_\n$
depending on $n$  such that 
the algorithm is universally consistent  \cite{stechr08}, that is 
\[
 \lim_{\n\to\infty}\Prob{\err(\fzn)-\err(\re)>\eps}=0
 \] 

for all $\eps>0$.
Moreover,  we study the  selection properties of the algorithm and prove that, if 
$\I$ is the set of relevant variables and $\Izn$ the set estimated by our algorithm, then the following
consistency result holds
$$
\lim_{\n\to\infty} \Prob{\Izn \subseteq \I} = 1.
$$

\item Finally we provide an extensive empirical analysis both on simulated and benchmark data,
showing that the proposed algorithm (DENOVAS)  compares favorably and often outperforms
other algorithms. This is particularly evident when the function to be estimated is highly non linear.
The proposed method can take advantage of working in a rich, possibly infinite dimensional,
hypotheses  space given by a RKHS, to obtain better estimation and selection properties.
This is illustrated  in Figure \ref{fig:radial_intro}, 
where the regression function is a  nonlinear function of 2 of  20 possible input variables. 
With 100 training samples the algorithms we propose is the  only one able to correctly solve the problem 
among different linear and non linear additive models.
On real data our method outperforms  other methods on several data sets. In most cases, the performance 
of our method and regularized least squares (RLS) are similar. However our method  brings higher interpretability since it is able to 
select a smaller subset of relevant variable, while the estimator provided by RLS depends on all  variables.\\

\end{enumerate}

\begin{figure}[t]
\caption{
 Comparison of predictions for a radial function of 2 out of 20 variables (the 18 irrelevant variables are not shown in the figure). 
In the upper left plot is depicted the value of the function on the test points (left),
the noisy training points (center), the values predicted for the test points by our method (DENOVAS)  (right).
The bottom plots represent the values predicted for the test points by 
state-of-the-art algorithms based on additive models.
Left: Multiple kernel learning based on additive models using kernels. 
Center: COSSO, which is a higher order additive model based on ANOVA decomposition \cite{lin2006}.
Right: Hierarchical kernel learning  \cite{bach09}. }
\label{fig:radial_intro}
	 \begin{center}
		  \includegraphics[width=.8\linewidth]{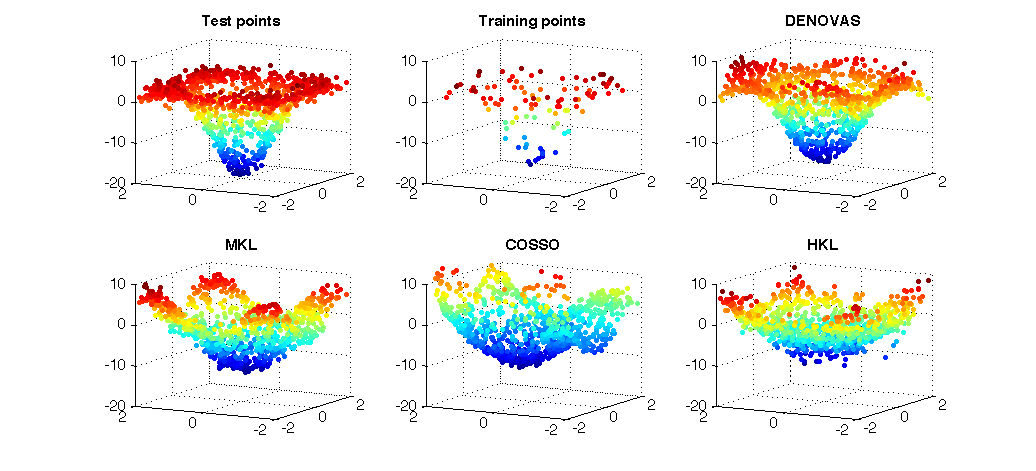}
	  \end{center}
\end{figure}

\section{Computational Analysis}\label{sec:computation}
In this section we study the minimization of the functional \eqref{algo_init}.

\subsection{Basic Assumptions}

We first begin by listing some basic conditions that we assume to hold throughout the paper.

We let $\rho$ be a probability measure on $\X\times\Y$ with  $\X\subset \rone^\dd$ and $\Y \subseteq \rone$. 
A training set  $\vz =(\vx,\vy)=(x_\ii, y_\ii)_{i=1}^n$ is a sample from $\rho^n$.
We consider a reproducing kernel $K:\X\times\X\to \rone$ \cite{aron50} and the associated reproducing kernel Hilbert space.
We assume $\rho$ and $K$ to satisfy the following assumptions.
\begin{itemize}
\item[\bf{[A1]}] {\it There exists $\ka<\infty$ such that $\sup_{x\in X} \norh{t\mapsto k(x,t)} < \ka.$} 

\item[ \bf{[A2]}] {\it The kernel $k$ is $\mathcal{C}^2(\X\times \X)$  and there exists $\kader<\infty$ such that for all $\jj=1,\dots,\dd$ we have
$\sup_{x\in X}\norh{t\mapsto \left. \frac{\partial k(s,x)}{\partial s^\jj}\right|_{s=t}} < \kader$\,.}

\item[\bf{[A3]}] {\it There exists $M<\infty$ such that $\Y\subseteq [-M,M].$}
\end{itemize}

\subsection{Computing the regularized solution}

We start our analysis discussing how to compute  efficiently a regularized solution of the functional

\begin{equation}\label{algo_computation}
\emp^{\pone}(f):=\frac{1}{\n}\sum_{i=1}^\n(y_\ii - f(x_\ii))^2+\pone \left(2\Regn(f)+\taumu\norh{f}^2\right),
\end{equation}
where $\Regn(f)$ is defined in \eqref{penalty}.
We start observing that the term  $\norh{f}^2$ makes the above functional coercive and strongly convex 
with modulus\footnote{We say that a function $\mathcal{E}:\hh\to\mathbb{R}\cup\{+\infty\}$ is: \begin{itemize}\item {\em coercive} if $\lim_{\nor{f}\to +\infty}{\mathcal{E}(f)}/{\nor{f}}=+\infty$; \item {\em strongly convex of modulus $\mu$} if $\mathcal{E}(t f+(1-t)g)\leq t\mathcal{E}(f)+(1-t)\mathcal{E}(g)-\frac{\mu}{2}t(1-t)\nor{f-g}^2$ for all $t\in[0,1]$.\end{itemize}} $\pone\taumu/2$,
so that standard results (\cite{ekeltem}) ensures existence and uniqueness of a minimizer $\fz$, for any $\nu>0$.

The rest of this section is divided into two parts. First we show how the theory of RKHS \cite{Aronszajn:1950} allows
to compute derivatives of functions on high dimensional spaces and also to derive a new representer theorem
that allows to deal with finite dimensional minimization problems. Second we discuss how to apply 
proximal methods \cite{combettes,beck09} to derive an iterative optimization procedure for which we can prove convergence.
It is possible to see that  the solution of Problem \eqref{algo_computation}
can be written as 
 \beeq{finite_rep_init}{
  \fz(x)= \sum_{\ii=1}^n \frac 1 \n \alpha_\ii k_{x_\ii}(x)+
  \sum_{\ii=1}^n \sum_{\jj=1}^\dd  \frac 1 \n \beta_{\jj,\ii} (\derh_\jj k)_{x_\ii}(x),
  }
 where  $\alpha, (\beta_{\jj,\ii})_{\ii=1}^\n\in \rone^\n$ for all $\jj=1, \dots, \dd$ $k_x$ is the function $t\mapsto k(x,t)$, and   $ (\derh_\jj k)_x$ denotes partial 
 derivatives of the kernel, see \eqref{der_repres}.
 The main outcome of our analysis  is that the coefficients $\alpha$ and $\beta$ can be provably computed through an iterative procedure.
To describe the algorithm we need some notation.
For all $\jj,\jjj = 1,\dots,\dd$, we define the $\n\times\n$ matrices $\matK, \matZ_\jj, \matL_{\jj,\jjj}$ 
as
\beeq{defK}{
\matK_{\ii,\iii}=\frac{1}{\n}k(x_\ii,x_\iii),
}
\beeq{defZ}{
[\matZ_\jj]_{\ii,\iii}=
\frac{1}{\n}\left.\frac{\partial k(s,x_\iii)}{\partial s^\jj}\right|_{s=x_\ii},
}
and
\[
[\matL_{\jj,\jjj}]_{\ii,\iii}=
\frac{1}{\n} \left.\frac{\partial^2 k(x,s)}{\partial x^\jj\partial s^\jjj}\right|_{x=x_\ii,s=x_\iii}\]
for all $\ii,\iii = 1,\dots,\n$.
Clearly the above quantities can be easily computed as soon as we have an explicit expression of the kernel, see Example \ref{DerExample} in Appendix \ref{tools_rkhs}.
We introduce also the  $\n\times\n\dd$ matrices
\[
\matZ = (\matZ_1,\dots,\matZ_\dd)\]
\beeq{defL}{\matL_\jj =(\matL_{\jj,1},\dots,\matL_{\jj,\dd}) \qquad \forall\jj=1,\dots,\dd}
and  the  $\n\dd\times\n\dd$ matrix 
$$
\matL = \left(\begin{array}{ccc}
\matL_{1,1}&\dots &\matL_{1,\dd}\\
\dots&\dots &\dots\\
\matL_{\dd,1}&\dots &\matL_{\dd,\dd}
\end{array}\right)
=\left(\begin{array}{c}
\matL_{\jj}\\
\dots\\
\matL_{\dd}
\end{array}\right)
$$
Denote with $\Bn$ the unitary ball in $\rone^\n$,
\beeq{Bn}{
\Bn=\{v\in\rone^\n~| \nor{v}_{\n}\le 1\}.
}
The coefficients in \eqref{finite_rep_init} are obtained through  Algorithm~\ref{algo:DENOVAS}, 
where $\beta$ is considered as a $\n\dd$ column vector $\beta = (\beta_{1,1},\dots,\beta_{1,\n},\dots,\beta_{\dd,1},\dots,\beta_{\dd,\n})^T$.

  \begin{algorithm}[!h]
   \caption{\  }
   \label{algo:DENOVAS}
   \begin{algorithmic}
    \STATE \textbf{Given:} parameters $\pone,\taumu>0$ and step-sizes $\step,\eta>0$\\
    \STATE \textbf{Initialize:} $\alpha^0 = \tilde{\alpha}^1 = 0$, $\beta^0= \tilde{\beta}^1 =0$, $s_1=1$, 
    $\bar{v}^0 = 0$, $\itext=1$\\
    \WHILE{\texttt{convergence not reached}}
    \STATE
    $ \itext = \itext+1$
    	\beeq{update_s}{
		s_{\itext} =\frac{1}{2}\lp1+\sqrt{1+4s_{\itext-1}^2}\rp
    }
         	\beeq{update_tilde}{
	    \tilde{\alpha}^{\itext} = \lp 1 +\frac{s_{\itext-1}-1}{s_{\itext}}\rp  \alpha^{\itext-1} + \frac{1-s_{\itext-1}}{s_{\itext}}\alpha^{\itext-2},\quad
	    \tilde{\beta}^{\itext} = \lp 1 +\frac{s_{\itext-1}-1}{s_{\itext}}\rp  \beta^{\itext-1} + \frac{1-s_{\itext-1}}{s_{\itext}}\beta^{\itext-2},\quad
    	}

     \beeq{update_a}{
    \alpha^\itext= \lp1-\frac{\ptwo \taumu}{\step} \rp \tilde{\alpha}^{\itext}- 
    \frac{1}{\step}\lp \matK\tilde{\alpha}^{\itext}+\matZ\tilde{\beta}^{\itext} -\vy \rp
    }
    \STATE \textbf{set} $ v^0 = \bar{v}^{\itext-1}$, $\itint=0$
    \WHILE{\texttt{convergence not reached}}
   \STATE  $\itint=\itint+1$
    \FOR{$\jj = 1,\dots \dd$\quad}
    \STATE
    \beeq{update_v}{
    v_\jj^{\itint} = \pi_{\frac{\pone}{\step} \Bn}\left(
    	v_\jj^{\itint-1}- \frac{1}{\eta}\left(
		\matL_\jj v^{\itint-1}-\lp\matZ_\jj^T\alpha^\itext + \lp1-\frac{\ptwo \taumu}{\step} \rp\matL_\jj\tilde{\beta}^{\itext}\rp
	    \right)
    \right)
    }

    \ENDFOR
    \ENDWHILE
    \STATE \textbf{set} $\bar{v}^\itext= v^\itint$
    	\beeq{update_b}{
    	 \beta^{\itext}=\lp1-\frac{\ptwo \taumu}{\step} \rp \tilde{\beta}^{\itext}- \bar{v}^{\itext}.
    	}   
    \ENDWHILE
    \RETURN $(\alpha^\itext, \beta^\itext)$
   \end{algorithmic}
  \end{algorithm}

The proposed optimization algorithm consists of two nested iterations, and involves
only matrix multiplications and thresholding operations. 
Before describing its derivation and discussing its convergence properties, 
  we add three  remarks. 
First, the proposed procedure requires the choice of an appropriate
 stopping rule, which will be discussed later, and of the step sizes
  $\step$ and $\eta$. 
  The simple a priori choice $\step = \nor{\matK}+\pone\taumu$, $\eta = \nor{\matL}$ 
  ensures convergence, as discussed in the Subsection \ref{sec:convergence},
   and is the one used in our experiments. 
  Second, the computation of the solution for different
  regularization parameters can be highly accelerated by a simple warm starting
  procedure, as the one  in \cite{hale08}.
Finally,
in Subsection \ref{sec:alg_issues} 
we discuss a principled way to select variable using
the norm of the coefficients $(\bar{v}^\itext_\jj)_{\jj=1}^\dd$.

\subsection{Kernels, Partial Derivatives and Regularization}
We start discussing how (partial) derivatives can be efficiently computed in RKHSs induced by smooth kernels and hence derive a new representer theorem.
Practical  computation of   the derivatives for a differentiable  functions is often performed 
via finite differences. For functions defined on a high dimensional space  such a procedure becomes 
cumbersome and ultimately not-efficient. RKHSs provide an alternative computational scheme.

Recall that the RKHS associated to a symmetric positive definite function $\kk:\X\times\X\to \rone$
is the unique Hilbert space $(\hh, \scalh{\cdot}{\cdot})$ such that $\kk_x=\kk(x,\cdot)\in \hh$, for all $x\in X$ and 
\begin{equation}\label{repro_prop}
f(x)=\scal{f}{\kk_x}_\hh, 
\end{equation}
for all $f\in \hh,x\in X$.
 Property \eqref{repro_prop} 
is called  {\em reproducing property} and $\kk$ is called reproducing kernel  \cite{Aronszajn:1950}.
We recall a few basic facts.
The functions in $\hh$ can be written as pointwise limits of finite linear combinations of the type $\sum_{i=1}^p \alpha_i \kk_{x_i}$, where $\alpha_i\in\rone, x_i \in X$ for all $i$. 
One of the most important results for kernel methods, namely the representer theorem \cite{wahba1990}, 
shows that  a large class of regularized kernel  methods induce estimators that can be written 
as {\em finite} linear combinations of kernels centered at the training set points. 
In the following we will make use of the so called {\em sampling operator}, which  returns the values of a function $f \in \hh$ 
at a set of input points $\vx=(x_1, \dots, x_n)$
\beeq{samp_op}{
\Sx:\hh\to \rone^n, \quad \quad  (\Sx f)_i=\scal{f}{\kk_{x_i}}, \quad i=1, \dots, n.
}
 The above operator is linear and bounded if the kernel is bounded-- see Appendix \ref{tools_rkhs}, which is true thanks to Assumption (A1).

Next, we discuss how the theory of RKHS allows efficient derivative computations.
Let
\beeq{der_repres}{
(\derh_\jj k)_x:= \left.\frac{\partial k(s,\cdot)}{\partial s^\jj}\right|_{s=x}
}
be the partial derivative of the kernel with respect to the first variable.
Then, from Theorem 1 in \cite{zhou08} we have that, if $\kk$ is  at least a $\mathcal{C}^2(\X\times \X)$,
 $(\derh_\jj k)_x$ belongs 
to $\hh$ for all $x\in X$ and most importantly 
$$
\frac{\partial f(x)}{\partial x^\jj}=
\scalh{f}{(\derh_\jj k)_x}, 
$$
for $\jj=1, \dots, \dd$, $x\in X$.
It is useful to define the analogous of the sampling operator for derivatives, which 
returns the values of the partial derivative of a function $f \in \hh$
at a set of input points $\vx=(x_1, \dots, x_n)$, 
\beeq{emp_der}{
\dern_\jj:\hh\to\rone^\n,\qquad  (\dern_\jj f )_\ii =  \scal{f}{(\derh_\jj k)_{x_\ii}},
}
where $\jj=1, \dots, \dd$, $i=1, \dots, n$. It is also useful to  define an empirical  gradient operator
$
\gradn:\hh\to (\rone^\n)^\dd$ defined by $\gradn f = (\dern_\jj f)_{\jj=1}^\dd.$
The above operators are linear and bounded, 
since assumption [A2] is satisfied. 
\noindent We refer to Appendix \ref{tools_rkhs} for further details and supplementary results.

Provided with the above results we can prove a suitable generalization of the representer theorem.
 \begin{proposition}\label{lemma:repr_theo}
The minimizer of \eqn{algo_computation} can be written as 
$$
\fz=  \sum_{\ii=1}^\n \frac 1 \n\alpha_\ii k_{x_\ii}+
  \sum_{\ii=1}^\n \sum_{\jj=1}^\dd \frac 1 \n \beta_{\jj,\ii}(\derh_\jj k)_{x_\ii}
 $$
  with $\alpha \in \rone$ and $\beta  \in \rone^{n\dd}$.
 \end{proposition}
\noindent The above result is proved in Appendix \ref{tools_rkhs} and shows that the regularized solution
is  determined by the set of $\n+\n\dd$ coefficients $\alpha\in\rone^\n$ and $\beta \in\rone^{\n\dd}$.
We next discuss how such coefficients can be efficiently computed. 
\\

\noindent{\bf Notation}. In the following, given an operator $A$ we denote by $A^*$ the corresponding adjoint operator.
When $A$ is a matrix we use the standard notation for the transpose $A^T = A^*$.

\subsection{Computing the Solution with Proximal Methods}\label{sec:derivation}

The functional $\emp^{\pone}$  is not differentiable, hence its minimization cannot be 
done by simple gradient methods. Nonetheless it has a special structure that allows efficient 
computations using a forward-backward splitting algorithm \cite{combettes}, belonging to the class of the so called proximal methods.  

Second order methods, see for example  \cite{ChaGolMul99}, could also  be used  to solve similar problems.
These methods typically   converge  quadratically and allows accurate computations.
However, they usually  have a high  cost per iteration and hence are  not suitable   for large scale problems,
 as opposed to   first order methods having much lower cost per iteration. 
Furthermore,  in the seminal paper by Nesterov  \cite{nesterov83}  first-order methods with optimal 
convergence rate are proposed \cite{nemirovski1983}. 
First order methods have since become a popular tool to solve  non-smooth problems in machine learning as well as   
signal and image processing, see  for example {\bf FISTA }-- \cite{beck09} and references therein. 
These methods have  proved to be fast and accurate \cite{becker09}, both for  $\ell^1$-based regularization --
see~\cite{combettes},~\cite{daubechies07},~\cite{figueiredo2007},~\cite{Lor2009b} -- 
and more general regularized learning methods -- see for example~\cite{duchi2009},~\cite{mosci2010ecml},~\cite{jenatton10} --.  \\

\paragraph{Forward-backward splitting algorithms}
 The functional $\emp^{\pone}$  is  the sum of the two terms  $F(\cdot)=\emp(\cdot)+\pone\taumu \nor{\cdot}^2_\hh$ and $2\pone\Regn$.
 The first term is strongly convex of modulus $\pone\taumu$ and  differentiable, while the second term  is convex but not differentiable. 
 The minimization of this class  of functionals can be done iteratively using  the  forward-backward (FB) splitting algorithm,
\begin{align}\label{basic_prox}
f^\itext &=
\text{prox}_{\frac{\pone}{\sigma} \Regn}
\Big(
\tilde{f}^\itext-\frac{1}{2\sigma} \nabla F(\tilde{f}^\itext)
\Big)\\
\label{acc_prox}\tilde{f}^\itext&=c_{1,t}f^{\itext-1}+c_{2,t}f^{\itext-2}
\end{align}
where  $f^0=f^1\in \hh$ is an arbitrary initialization, $c_{1,t}, c_{2,t}$ are suitably chosen positive sequences,  and $\mathrm{prox}_{\frac{\pone}{\sigma}\Regn}:\hh\to\hh$ is the proximity operator \cite{Mor65} defined by, 
\[
\mathrm{prox}_{\frac{\pone}{\sigma}\Regn}(f)=\argmin_{g\in\hh}\lp \frac{\pone}{\sigma} \Regn(g)+\frac 1 2 \nor{f-g}^2 \rp.
\] 
The above approach decouples the contribution of the differentiable and not differentiable terms. 
Unlike other simpler penalties used in additive models, such as the $\ell^1$ norm in the lasso, 
in our setting the computation of the proximity operator of $\Regn$ is not trivial and will be discussed in the next paragraph.
Here we briefly recall the main properties  of the iteration \eqref{basic_prox}, \eqref{acc_prox}
depending on the choice of  $c_{1,t},c_{2,t}$ and $\sigma$.
The basic version   of the algorithm \cite{combettes}, sometimes called ISTA (iterative shrinkage thresholding algorithm \cite{beck09}),  is obtained  setting $c_{1,t}=1$ and $c_{2,t}=0$ for all $t>0$, so that  each step depends only on the previous iterate. 
The convergence of the  algorithm for both the objective function values and the minimizers 
is  extensively studied in \cite{combettes}, but a convergence rate is not provided. 
In \cite{beck09} it is shown that  the convergence of the objective function values is of order $O(1/t)$ 
provided that the step size $\step$  satisfies $\step\geq L$, where $L$ is the Lipschitz constant of $\nabla F/2$. 
An alternative choice of $c_{1,t}$ and $c_{2,t}$ leads to an accelerated version of  the algorithm \eqref{basic_prox},  sometimes called FISTA (fast iterative shrinkage thresholding algorithm \cite{Tse10, beck09}), which is  obtained by setting $s_0=1$,
\beeq{eq:ct}{
s_{\itext} =\frac{1}{2}\lp1+\sqrt{1+4s_{\itext-1}^2}\rp,
\qquad c_{1,t} = 1+\frac{s_{\itext-1}-1}{s_{\itext}}, \text{~~and~~~} c_{2,t} =\frac{1-s_{\itext-1}}{s_{\itext}}.}
The algorithm is analyzed in \cite{beck09} and in \cite{Tse10} where it is proved that the objective values generated by such a procedure have convergence of  order  $O(1/\itext^2)$, if the step size satisfies $\step\geq L$.
\\

Computing the Lipscthitz constant $L$ can be  non trivial.
Theorems 3.1 and 4.4 in \cite{beck09} show  that the iterative procedure \eqref{basic_prox}
with an adaptive choice for the step size, called {\em backtracking}, which does not require the computation of  $L$,
shares the same rate of convergence of the corresponding procedure with  fixed step-size.
Finally, it is well known that, if the functional is strongly convex with a positive modulus,
the convergence rate of both the basic and accelerated scheme 
is indeed linear  for both the function values and the minimizers \cite{nesterov83,mosci2010ecml,nesterov07}.
\\

In our setting we use FISTA to tackle the minimization of $\emp^{\pone}$ but, as we mentioned before, 
we have to deal with the  computation of the proximity operator associated to $\Regn$. 
\paragraph{Computing the proximity operator.} 
Since $\Regn$ is one-homogeneus, i.e. $\Regn(\lambda f)=\lambda\Regn(f)$ for $\lambda>0$, 
the Moreau identity, see \cite{combettes}, gives a useful alternative formulation for the proximity operator, that is
\beeq{eq:Moreau}{
\mathrm{prox}_{\frac{\pone}{\sigma}\Regn}=I-\pi_{\frac{\tau}{\sigma}\Cn},
}
where $\Cn=(\partial\Regn)(0)$ is the subdifferential
\footnote{
Recall that the subdifferential of a convex functional $\Omega:\hh\to \rone\cup\{+\infty\}$ is denoted with $\partial \Omega (f)$ and
is defined as the set
$$
\partial \Omega (f) := \{h \in \hh\,:\, \Omega(g)-\Omega(f)\geq \scalh{h}{g-f},~~ \forall g\in\hh\}.
$$
}
 of $\Regn$ at the origin, and $\pi_{\frac{\tau}{\sigma}\Cn}:\hh \to \hh$ is the projection on $\frac{\tau}{\sigma}\Cn$-- 
which is well defined since $\Cn$ is a closed convex subset of $\hh$. 
To describe how to practically compute such a projection, we start observing  that the DENOVAS penalty $\Regn$   is the sum of $\dd$ norms in $\rone^\n$.
Then following Section $3.2$ in \cite{mosci2010ecml} (see also \cite{ekeltem}) we have 
$$
  \Cn= \partial \Regn(0)=\lb  f\in \hh ~|~ f=\gradn^*v ~\text{with}~  v \in \Bnd  \rb,
 $$
where $\Bnd$ is the cartesian product of $\dd$ unitary balls in $\rone^\n$, 
$$
\Bnd=\underbrace{\Bn\times\dots\times\Bn}_{\dd~ \textrm{times}}=\{v=(v_1,\dots, v_\dd)~| v_\jj\in\rone^\n,~ \nor{v_\jj}_{\n}\le 1, ~~\jj=1,\dots,\dd \},
$$
with  $\Bn$ defined in \eqref{Bn}. 
Then, by definition, the projection  is  given by
$$
\pi_{\frac{\pone}{\sigma}\Cn} (f)=\gradn^*\bar{v},
$$
where 
\beeq{vbar}{
\bar{v}\in \argmin_{v\in\frac{\pone}{\sigma}\Bnd}   \nor{f-\gradn^*v}_\hh^2.
}
Being a convex constrained problem, \eqref{vbar} can be seen as the sum of the smooth term  $\norh{f-\gradn^*v}^2$ and the indicator function of the convex set $\Bnd$.    We can therefore   use \eqref{basic_prox}, again.
In fact  we can fix an arbitrary initialization  $v^0\in\rone^{\n\dd}$ and consider,
\beeq{eq:vq}
{v^{\itint+1} = \pi_{{\frac{\pone}{\sigma}} \Bnd}\left(v^{\itint}-\frac{1}{\eta}\gradn(\gradn^* v^\itint - f)\right),}
for a suitable choice of $\eta$. In particular, we notice that $ \pi_{{\frac{\pone}{\sigma}} \Bnd}$ can be easily computed in closed form, and corresponds to the proximity operator associated to the indicator function of $\Bnd$.
Applying the results mentioned above, if $\eta\geq \nor{\gradn\gradn^*}$, convergence of the function values of problem \eqref{vbar} on the sequence generated via \eqref{eq:vq} is guaranteed. However, since we are interested in the computation of the proximity operator, this is not enough. Thanks to the special structure of the minimization problem in \eqref{vbar},  it is possible to prove (see \cite{ComDunVu10,mosci2010ecml}) that 
\beeq{lim-proj}{
 \nor{\gradn^*v^\itint-\gradn^*\bar{v}}_\hh\to 0, \quad \text{ or, equivalently } \quad \nor{\gradn^*v^\itint- \pi_{\frac{\pone}{\sigma}\Cn} (f)}_\hh\to 0.
}

A similar first-order method to compute convergent approximations of $\gradn^*\bar{v}$ has been proposed in \cite{BecBlaAubCha04}.

\subsection{Overall Procedure and Convergence analysis} \label{sec:convergence}

To compute the minimizer of  $\emp^{\pone}$ we consider the
combination of the accelerated  FB-splitting algorithm (outer iteration) 
and the basic FB-splitting algorithm for computing the proximity operator (inner iteration). 
The overall procedure is given by
\begin{eqnarray}\label{itext}
	s_{\itext} &=&\frac{1}{2}\lp1+\sqrt{1+4s_{\itext-1}^2}\rp\nonumber\\
	\tilde{f}^{\itext} &=& \lp1+\frac{s_{\itext-1}-1}{s_{\itext}}\rp f^{\itext-1} + \frac{1-s_{\itext-1}}{s_{\itext}}f^{\itext-2}\\
	f^{\itext}&=&  \lp1-\frac{\ptwo \taumu}{\step} \rp \tilde{f}^\itext-  \frac{1}{\sigma} \Sx^* \lp  \Sx \tilde{f}^\itext - \vy \rp - \gradn^* \bar{v}^\itext,\nonumber
\end{eqnarray}
for $\itext = 2, 3,\dots$, where $\bar{v}^\itext$ is computed through the iteration
\beeq{itint}{
v^{\itint} = \pi_{{\frac{\pone}{\step}} \Bnd}\left(v^{\itint-1}-\frac{1}{\eta}\gradn\lp\gradn^* v^{\itint-1} - \lp1-\frac{\ptwo \taumu}{\step} \rp \tilde{f}^\itext-  \frac{1}{\sigma} \Sx^* \lp  \Sx \tilde{f}^\itext - \vy \rp \rp\right),
}
for given  initializations.

%

The above  algorithm is an {\em inexact} accelerated FB-splitting algorithm,
in the sense that  the proximal or backward step is computed only approximately.
The above discussion on the convergence of FB-splitting algorithms was limited to the 
case where computation of the proximity operator is done exactly (we refer to this case as the {\em exact} case).
The  convergence of the  inexact FB-splitting algorithm 
does not follow from this analysis.
For the basic -- not accelerated -- FB-splitting algorithm,
convergence in the inexact case is still guaranteed (without a rate) \cite{combettes},
if the computation of the proximity operator is sufficiently accurate.
The convergence of the  inexact accelerated FB-splitting algorithm 
 is studied in  \cite{salzo2011prox} where it is shown that the same convergence rate of the exact case can be achieved, 
again provided that the accuracy in the computation of the proximity operator can be suitably controlled.
Such a result can be adapted to our setting to prove the following theorem,  as shown in Appendix \ref{app:computation}.

\begin{theorem}\label{theo:convergence}
Let $\varepsilon^t\sim t^{-l}$ with $l>3/2$, $\step\geq\nor{\Sx^*\Sx} +\ptwo\taumu$, $\eta\geq\nor{\gradn\gradn^*}$,
and  $f^\itext$ given by \eqref{itext} with $\bar{v}^\itext$ computed through \eqref{itint}. Define $g^\itext= \lp1-\frac{\ptwo \taumu}{\step} \rp \tilde{f}^\itext-  \frac{1}{\sigma} \Sx^* \lp  \Sx \tilde{f}^\itext - \vy \rp$.
If $\bar{v}^\itext=v^\itint$, for $\itint$ such that the following condition is satisfied
\beeq{eq:init}{
\frac{2\tau}{\sigma} \Regn(f^\itext)-2 \langle \gradn^*v^{\itint},f^\itext\rangle \leq\varepsilon^{2\itext}\,,}
Then there exists a constant $C>0$ such that
$$
\emp^{\pone}(f^\itext)-\emp^{\pone}(\fz) \leq \frac{C}{\itext^2},
$$
and thus, if $\taumu>0$,
\beeq{lim-f}{
\nor{f^\itext-\fz}_\hh\leq  \frac{2}{  \itext} \sqrt{\frac{C}{\taumu\pone}}.
}
\end{theorem}

As for the exact accelerated FB-splitting algorithm, 
the step size of  the outer iteration  has to be greater than or equal to $L=\nor{\Sx^*\Sx} +\ptwo\taumu$.
In particular, we choose  $\step=\nor{\Sx^*\Sx} +\ptwo\taumu$ and, similarly, $\eta = \nor{\gradn\gradn^*}$.

We add few remarks.
First, as it is evident from \eqref{lim-f}, the choice of $\taumu>0$ allows to obtain convergence 
of $f^\itext$ to $\fz$ with respect to the norm  in $\hh$, 
and positively influences the rate of convergence.
This is a crucial property in variable selection, where it is necessary to accurately estimate the minimizer of the expected risk $\fdag$
and not only its minimum $\err(\fdag)$. 
Second, condition \eqref{eq:init} represents an implementable stopping criterion for the inner iteration, once that the representer theorem is proved.
Further comments on the stopping rule are given in  Section \ref{sec:alg_issues}.
Third, we remark that for proving convergence of the inexact procedure, 
it is essential that the specific algorithm proposed to compute the proximal step
generates a sequence belonging to $\Cn$ and satisfying \eqref{lim-proj}.

 \subsection{Further Algorithmic Considerations} \label{sec:alg_issues}
 We conclude discussing several  practical aspects of the proposed method. 
 
 \paragraph{The finite dimensional implementation.}
 
 We start by showing how the representer theorem can
be used, together with the iterations   described by  \eqref{itext} and \eqref{itint}, to derive  Algorithm \ref{algo:DENOVAS}.
 This is summarized in the following  proposition. 
\begin{proposition}\label{prop:itext}
For $\taumu> 0$ and
$f^0 = \frac 1 \n\sum_{\ii} \alpha^0_\ii k_{x_{\ii}}+
\frac 1 \n  \sum_{\ii} \sum_{\jj} \beta^0_{\jj,\ii}(\derh_\jj k)_{x_{\ii}}$
for any $\alpha^0\in\rone^\n, \beta^0\in\rone^{\n\dd}$,
the solution at step $\itext$ for the updating rule \eqref{itext} is given by
\beeq{repr_itext}{
f^{\itext} =  
\frac 1 \n\sum_{\ii=1}^\n \alpha^\itext_\ii k_{x_{\ii}}+
\frac 1 \n  \sum_{\ii=1}^\n \sum_{\jj=1}^\dd \beta^\itext_{\jj,\ii}(\derh_\jj k)_{x_{\ii}}
}
 with $\alpha^\itext$ and $\beta^\itext$ defined by the updating rules 
   (\ref{update_a}-\ref{update_tilde}),
where $\bar{v}^{\itext}$ in \eqref{update_b} 
 can be estimated, starting from  any $v^0\in\rone^{\n\dd}$, and  using the  iterative rule \eqref{update_v}. 
  \end{proposition}
The proof of the above proposition can be found in  Appendix \ref{app:computation}, and is based on the  observation that 
 $\matK, \matZ_\jj, \matZ, \matL_\jj$ defined at the beginning of this  Section
are the matrices associated to  the  operators
$\Sx \Sx^*: \rone^\n \to \rone^\n$, $\Sx \dern_\jj^*: \rone^\n \to \rone^\n$,
$\Sx \gradn^*: \rone^{\n\dd} \to \rone^{\n}$
and $\dern_\jj\gradn^*:\rone^{\n\dd}\to \rone^\n$, respectively. 
Using the same reasoning we can make the following two further  observations.
First, one can compute the step sizes $\step$ and $\eta$
as $\step =  \nor{\matK}+\pone\taumu $, 
 and $\eta = \nor{\matL}$.
Second, since in practice we have to define  suitable stopping rules,
  Equations \eqref{lim-f} and \eqref{lim-proj}    suggest the following choices
  \footnote{In practice we often  use a stopping rule where the tolerance is scaled with the current iterate,
 $
  \nor{f^\itext-f^{\itext-1}}_{\hh}\le \varepsilon^{\text{(ext)}} \nor{f^\itext}_\hh$ and 
  $\frac{2\tau}{\sigma} \Regn(f^\itext)-2 \langle \gradn^*v^{\itint},f^\itext\rangle \leq\varepsilon^{\text{(int)}}\nor{\gradn^* v^\itint}_\hh.
  $}
  $$
  \nor{f^\itext-f^{\itext-1}}_{\hh}\le \varepsilon^{\text{(ext)}} ~~~ \text{
  and} ~~~
 \frac{2\tau}{\sigma} \Regn(f^\itext)-2 \langle \gradn^*v^{\itint},f^\itext\rangle \leq\varepsilon^{(\text{int})}\,.
  $$
As a direct consequence of  \eqref{repr_itext} and using the definition of matrices $\matK,\matZ,\matL$, 
 these quantities can be easily computed  as
\begin{eqnarray*}
   \nor{f^\itext-f^{\itext-1}}_\hh^2  
&    =&  \scal{\delta\alpha}{\matK\delta\alpha}_\n +2\scal{\delta\alpha}{\matZ\delta\beta}_\n +\scal{\delta\beta}{\matL\delta\beta}_\n\,,\\
 \frac{2\tau}{\sigma} \Regn(f^\itext)-2 \langle \gradn^*v^{\itint},f^\itext\rangle &=&
\sum_{\jj=1}^\dd \nor{Z_\jj\alpha^\itext+L_\jj \beta^\itext}  -2 \scal{\alpha^\itext}{\matZ v^\itint}_\n\,.
  \end{eqnarray*}
 where we defined
  $\delta\alpha= \alpha^\itext-\alpha^{\itext-1}$ and  $\delta\beta=\beta^\itext-\beta^{\itext-1}$.  Also note  that, according to Theorem \ref{theo:convergence}, $\varepsilon^{\text{(int)}}$ must depend on the outer iteration as $\varepsilon^{\text{(int)}} = \varepsilon^{\itext}\sim\itext^{-2l}$,  $l>3/2$.

Finally we discuss a criterion for identifying the variables selected by the algorithm. 
\paragraph{Selection.}
Note that in the linear case $f(x)=\beta\cdot x $ the coefficients $\beta^1,\dots,  \beta^\dd$ coincide with 
the partial derivatives, and  the coefficient vector $\beta$  given by  $\ell^1$ regularization is sparse (in the sense that it has zero entries), 
so that it is easy to detect  which variables are to be considered relevant. 
For a general non-linear function, we then expect the vector $(\|\dern_\jj f\|_\n^2)_{\jj=1}^\dd$ of the 
norms of the partial derivatives evaluated on the training set points, to be sparse as well. 
In practice since the projection $\pi_{\pone/\sigma \Bnd}$ is computed only 
approximately,  the norms of the partial derivatives  will be small but typically not zero. The following proposition elaborates on this point.
\begin{proposition}\label{prop:eulero}
Let $v = (v_\jj)_{\jj=1}^\dd\in \Bnd$ such that, for any $\step>0$ 
$$
\gradn^* v = 
- \frac{1}{\step}\nabla (\emp(\fz)  + \pone\taumu\norh{\fz}^2),
$$
then
\beeq{eq:selection}{
\nor{v_\jj}_\n <\frac{\pone}{\sigma}\Rightarrow \nor{ \dern_\jj \fz}_\n=0.
}
Moreover, if $\bar{v}^\itext$ is given by Algorithm \ref{algo:DENOVAS} with the inner iteration stopped
when the assumptions of Theorem \ref{theo:convergence} are met, 
 then there exists $\tilde{\eps}^\itext>0$  (precisely defined in \eqref{eq:epstilde}) depending on the tolerance $\varepsilon^{\itext}$ used in the inner iteration
 and satisfying $\lim_{\itext\to 0} \tilde{\eps}^\itext = 0$, such that if $m:=\min\{\nor{ \dern_\jj \fz}_\n\,:\, \jj\in\{1,\ldots,d\} \mathrm{s. t. } \nor{ \dern_\jj \fz}_\n>0\}.$
\beeq{eq:vbarok}
{\nor{\bar{v}_\jj^\itext}_\n\geq \frac{\pone}{\step}-\frac{(\tilde{\eps}^\itext)^2}{2m} \Rightarrow \nor{ \dern_\jj \fz}_\n=0 .}
\end{proposition}
The above result, whose proof can be found in Appendix \ref{app:computation},  
is a direct consequence of the Euler equation for $\emp^{\pone}$
and of the characterization of the subdifferential of $\Regn$.
The second part of the proof follows by observing that, as 
$\gradn^* v$
 belongs to the 
subdifferential of $\Regn$ at $\fz$,
$\gradn^*\bar{v}^\itext$ belongs to the 
{\em approximate} subdifferential of $\Regn$ at $\fz$, where
the approximation of the subdifferential is controlled by the precision used in evaluating the projection.
{Given the  pair $(f^\itext,\bar{v}^\itext)$ evaluated via Algorithm \ref{algo:DENOVAS}, 
we can thus consider to be irrelevant the variables such that $\nor{\bar{v}^\itext_\jj}_\n <\pone/\sigma -(\tilde{\eps}^\itext)^2/(2m)$.
Note that the explicit form of $\tilde{\eps}^\itext$ is given in  \eqref{eq:epstilde}). 

 \section{Consistency for Learning and Variable Selection}\label{sec:cons}

In this section we study the consistency properties of our method.

\subsection{Consistency} 
As we discussed in Section \ref{sec:main_explain}, though in practice we consider the regularizer $\Regn $ defined in \eqref{penalty}, 
ideally we would be interested into  $\Reg (f)= \sum_{\jj=1}^\dd  \nor{\der_\jj f}_{\marg}$, $~f\in \hh$.
The following preliminary result shows that indeed $\Regn$ is a consistent estimator of $\Reg$ when considering  functions in $\hh$
having uniformly bounded norm.

\begin{theorem} \label{teo:reg}
Let $r<\infty$, then under assumption (A2)
$$
 \lim_{\n\to\infty}\Prob{\sup_{\norh{f}\leq r}|\Regn(f)- \Reg(f)|> \epsilon}=0 \qquad \forall \epsilon>0.
 $$
\end{theorem}
The restriction to functions such that $\norh{f}\leq r$  is natural and is required since 
the penalty $\Regn$ forces the partial derivatives to be zero only on the training set points.
To guarantee that  a partial derivative,  which is zero on the training set,  
is also  close to zero on the rest of the input space, we must control the smoothness of the function class where the derivatives
are computed. This motivates constraining the function class by adding    the (squared) norm  in $\hh$ into \eqref{algo_init}. 
This is in the same spirit of the manifold regularization proposed in \cite{BelNiy08}. 

The above result on the consistency of the derivative based regularizer  is at the basis of the following consistency result.
\begin{theorem}\label{teo:consistency}
Under assumptions A1, A2 and A3, recalling that $\err(f)=\int(y-f(x))^2\,d\rho(x,y)$,
$$
 \lim_{\n\to\infty}\Prob{\err(\fzn)-\inf_{f \in \hh} \err(f) \geq\epsilon}=0\qquad \forall \epsilon>0,
 $$
 for any $\pone_\n$  satisfying
 $$
 \pone_\n\to0\qquad (\sqrt{\n}\pone_\n)^{-1}\to 0.
 $$
\end{theorem}
The proof is given in the appendix and is based 
on a  sample/approximation  error decomposition 
$$
\err(\fz)-\inf_{f \in \hh} \err(f)\le \underbrace{|\err(\fz) - \err^{\pone}(f^\tau)|}_{\text{sample error}}+
\underbrace{|\err^{\pone}(f^\tau) - \inf_{f \in \hh} \err(f)|}_{\text{approximation error}},
$$
where 
$$
\err^{\pone}(f):=\err(f)+2\pone\Reg(f)+\ptwo \taumu\nor{f}_\hh^2,\qquad f^\pone:=\argmin_{\hh}\err^{\pone}.
$$
 The control  of both terms allows to find a suitable 
parameter choice which gives consistency. 
When estimating the sample error one has typically to control only the deviation of the empirical risk from its continuos counterpart.
Here we need Theorem \ref{teo:reg}  to also control the deviation of  $\Regn$ from $\Reg$. 
Note that, if the kernel is universal \cite{stechr08}, 
then $\inf_{f\in \hh }\err(f) = \err(\re)$  and Theorem \ref{teo:consistency} gives the  universal consistency of the estimator $\fzn$.\\

 To study  the selection properties of the estimator $\fzn$-- see next section-- 
it useful to study the distance of  $\fzn$ to $\re$ in the $\hh$-norm.
Since in general $f_\rho$ might not belong to $\hh$, for the sake of generality here we 
compare $\fzn$ to a  minimizer of $\inf_{f\in \hh}\err(f)$ which we always assume to exist. 
Since the minimizers might be more then one we further consider a suitable  minimal norm minimizer $\fdag$-- see below.
More precisely  given the set $$\mathcal{F}_\hh :=\{f\in\hh ~|~ \err(f)= \inf_{f \in \hh} \err(f)\}$$ (which we assume to be not empty), 
we define 
$$
\fdag:= \argmin_{f \in\mathcal{F}_\hh} \{ \Reg (f) + \taumu\norh{f}^2\}.
$$
Note that  $\fdag$   is well defined and unique, since $\Reg(\cdot)  + \taumu\norh{\cdot}^2$ is  
strongly convex and  $\err$ is convex and lower semi-continuous on $\hh$,  which implies that $\mathcal{F}_\hh$ is closed and convex in $\hh$.
Then, we have the following result.
\begin{theorem}\label{teo:strong_consistency}
Under assumptions A1, A2 and A3, we have
$$
 \lim_{\n\to\infty} \Prob{ \nor{\fzn-\fdag}_\hh\geq \epsilon}=0,\qquad \forall \epsilon>0,
$$
for any $\pone_\n$ such that $ \pone_\n\to 0$ and $(\sqrt{\n}\pone^2_\n)^{-1}\to 0$.
\end{theorem}

The proof, given in Appendix \ref{app:cons},  is based
on the decomposition in sample error, $\| \fz - f^\tau\|_\hh$, 
and approximation error, $\|f^{\pone} - \fdag\|_\hh$.
To bound the sample error we use recent   results \cite{villa2010epi} that exploit Attouch-Wets convergence 
\cite{attwet191,attwet293,attwet393} 
and coercivity  of the penalty (ensured by the RKHS norm)
to control  the distance between the minimizers $\fz,f^{\pone} $ by the distance the minima 
 $\emp^{\pone}(\fz)$ and $\err^{\pone}(f^{\pone})$. Convergence of the approximation error 
is again guaranteed by standard results in regularization theory \cite{donzol93}. We underline that our result is an asymptotic one, 
although it would be interesting to get an explicit learning rate, as we discuss in Section \ref{sec:rates}.

\subsection{Selection properties}\label{sec:selection}
We next  consider the selection properties of our method.  
Following Equation \eqref{nonparzero}, we start by giving the definition of relevant/irrelevant variables and sparsity in our context. 

\begin{definition}\label{def:irrel}
We say that a variable $a=1, \dots, \dd$ is  {\em irrelevant} with respect to $\rho$ for a differentiable function $f$,
if the corresponding partial derivative 
$\der_\jj f$ is zero $\marg$-almost everywhere, and relevant otherwise. In other words the set of relevant variables is 
$$
R_f :=\{\jj\in\{1,\dots,\dd\} ~|~ \nor{\der_\jj f}_{\marg}>0\}.
$$
We say that a differentiable function  $f$ is sparse if $\Regze(f):=|R_f|< d$. 
\end{definition}
\noindent The goal of variable selection is to correctly estimate the set of relevant variables, $\I:=R_{\fdag}$.
In the following we study how this can be achieved by the empirical set of  relevant variables, $\Izn$, defined as
$$\Izn :=\{\jj\in\{1,\dots,\dd\} | \nor{\dern_\jj \fzn}_{\n}>0\}.$$

\begin{theorem}\label{teo:subset}
Under assumptions A1, A2 and A3
$$
\lim_{\n\to\infty} \Prob{\I\subseteq \Izn} = 1
$$
for any $\pone_\n$ satisfying 
$$
\pone_\n\to 0\qquad (\sqrt{\n}\pone_\n^2)^{-1}\to 0.
 $$
\end{theorem}
\noindent The above result shows that the proposed regularization scheme is a safe
filter for variable selection,  since it does not discard relevant variables, in fact,
for a sufficiently large number of training samples, 
the set of truly relevant variables, $\I$,  is  contained with high probability in the 
set of relevant variables identified by the algorithm, $\Izn$. 
The proof of the converse inclusion, giving consistency for variable selection (sometimes 
called sparsistency),  requires further analysis that we postpone to a future work 
(see the discussion in the next subsection).

\subsection{Learning Rates and Sparsity}\label{sec:rates}

The analysis in the previous  sections is asymptotic, so it is  natural to ask about  the finite sample 
behavior of the  proposed method, and in particular  about the implication  
of the sparsity  assumption. Indeed, for a variety of additive models it is possible to prove that 
the sample complexity (the number of samples needed to achieve a given error with a specified probability) 
depends linearly on the sparsity level and in a much milder way to the total number of variables, 
e.g. logarithmically \cite{buhvan11}. Proving similar results in our setting is considerably more complex and in this section 
we discuss the main sources of possible challenges. 

Towards this end, it is interesting to contrast the form of our regularizer  to that of structured sparsity penalties for which 
sparsity results can be derived.  
Inspecting the proof in Appendix \ref{app:cons}, one can see that it  possible to define a suitable family of operators
$V_j,\hat V_j:\hh\to \hh$, with $j=1, \dots,d$,  such that 
\begin{equation}\label{our}
\Reg(f) = \sum_{j=1}^d \norh{V_j f }, \quad \quad \Regn(f) = \sum_{j=1}^d \norh{\hat V _j f }.
\end{equation}
The operators $(V_j)_j$ are positive and self-adjoint and so are the operators $(\hat V_j)_j$.  
The latter  can be shown to be stochastic approximation   of the operators  $(V_j)_j$, in the sense that 
the equalities  $\E[\hat{V}_j^2]=V^2_j$ hold true for all $j=1, \dots,d$. 

It is interesting to compare the above expression to the one for the group lasso penalty, where  for a given linear model, 
the coefficients are assumed to be divided in groups, only few of which are predictive.
 More precisely, given  a collection  of groups of indices $\mathcal{G}=\{G_1,\ldots,G_r\}$,
which forms a partition of the set $\{1,\ldots,p\}$, and  a linear model $f(x)=\scal{\beta}{x}_{\rone^p}$, 
the group lasso penalty is obtained by considering 
$$\reg^{GL}(\beta)=\sum_{\gamma=1}^r \nor{ \beta_{|G_\gamma}}_{\mathbb{R}^{|G_\gamma|}},$$
where, for each $\gamma$,   $ \beta_{|G_\gamma}$ is the $|G_\gamma|$ dimensional vector
obtained restricting a vector $\beta$ to the indices in $G_\gamma$.
If we let $P_\gamma$ be the orthogonal projection on the subspace of $\rone^d$ 
corresponding  $G_\gamma$-th group of indices, we have that  $\scal{P_\gamma \beta}{P_\gamma' \beta'}=0$ for all $\gamma,\gamma' \in \Gamma$ and $\beta,\beta'\in\rone^p$, 
since the groups form a partition of $\{1,\ldots,p\}$. 
Then it is possible to rewrite the group lasso penalty as 
$$\reg^{GL}(\beta)=\sum_{\gamma=1}^r \nor{P_\gamma \beta}_{\mathbb{R}^{|G_\gamma|}}.$$
The above idea can be extended to an infinite dimensional setting to obtain multiple kernel learning (MKL).
Let $\hh$ be a (reproducing kernel) Hilbert space which is the sum of $\Gamma$  disjoint (reproducing kernel) 
Hilbert spaces $(\hh_\gamma, \nor{\cdot}_\gamma)_{\gamma \in \Gamma}$, and 
$P_\gamma:\hh\to \hh_\gamma$ the projections of $\hh$ onto $\hh_\gamma$, then  MKL is induced by the penalty 
$$
\reg^{MKL}(f)=\sum_{\gamma\in \Gamma}\nor{P_\gamma f}_\gamma. 
$$ 

When compared to our derivative based penalty, see~\eqref{our}, one can notice at least two source of difficulties:
\begin{enumerate}
\item  the operators $(V_j)_j$ are not projections and  no simple relation exists among their ranges,
\item in practice we have only access to the  empirical estimates $(\hat V_j)_j$.
\end{enumerate}

Indeed, structured sparsity model induced by  more complex index sets have been considered, 
see for example \cite{jenatton10}, but the penalties are still induced by operators which are orthogonal projections.
Interestingly, a  class of penalties induced by a (possibly countable) family of bounded 
operators ${\cal V}=\{V_\gamma\}_{\gamma \in \Gamma}$-- not necessarily projections-- 
has been considered in \cite{MauPon12}. This class of penalties
can be written as 
$$
\reg^{({\cal V})}(f)= \inf\{ \sum_{\gamma \in \Gamma} \nor{f_\gamma}~|~f_\gamma\in \hh ,  \sum_{\gamma \in \Gamma} V_\gamma f_\gamma=f \}.$$
It is easy to see that the above  penalty does not include the regularizer~\eqref{our} as a special case. 

In conclusion, rewriting our derivative based regularizer as in~\eqref{our} highlights similarity and differences 
with respect to previously studied sparsity methods: indeed many of  these methods are induced by families of 
operators. On the other hand, typically, the operators are assumed to satisfy stringent assumptions which do not 
hold true in our case. Moreover in our case one would have to overcome the difficulties arising from the 
random estimation of the operators. These interesting questions  are outside of the scope of this paper, will be the subject of future work. 

\section{Empirical Analysis} \label{sec:emp_val}

The content of this section is divided into three parts.
First, we describe the choice of tuning parameters. 
Second, we study the properties of the proposed method on simulated data under different parameter settings, and
third, we  compare our method to related regularization methods for learning and variable selection. 

When we refer to our method we always consider a two-step procedure 
based on variable selection via Algorithm \ref{algo:DENOVAS} and 
regression on the selected variable via (kernel) Regularized Least Squares (RLS).
The kernel used in both steps is the same.
Where possible, we applied the same reweighting procedure to the methods we compared with.

\subsection{Choice of tuning parameters}
When using Algorithm \ref{algo:DENOVAS}, once the parameters $n$ and $\nu$ are fixed, 
we evaluate the optimal value of the regularization parameter $\ptwo$ via hold out validation on 
an independent validation set of $nval = n$ samples. The choice of the parameter $\taumu$, and its
influence on the estimator is discussed in the next section.  

Since we use an iterative procedure to compute the solution, the output of our algorithm will not be sparse
in general and a selection criterion is needed. In Subsection \ref{sec:alg_issues} 
we discussed a principled way to select variable using the norm of the coefficients $(\bar{v}^\itext_\jj)_{\jj=1}^\dd$.

When using MKL, $\ell^1$ regularization, and RLS we used hold out validation to set the regularization parameters, while
for COSSO and  HKL we used the choices suggested by the authors. 

\subsection{Analysis of Our Method}

\subsubsection{Role of the smoothness enforcing penalty $\taumu\nor{\cdot}^2_{\hh}$}\label{sec:vario_nu}
From a theoretical stand point we have shown that $\taumu$ has to be nonzero, 
in order for the proposed regularization problem \eqref{algo_init} to be well-posed.
We also mentioned that 
the combination of the two penalties $\Regn$ and $\norh{\cdot}^2$
ensures that the regularized solution will not depend on variables that are irrelevant for two different reasons.
The first is irrelevance with respect to the output.
The second type of irrelevance is meant in an unsupervised sense.
This happens when one or more variables are (approximately) constant with respect to the marginal distribution $\marg$,
so that the support of the marginal distribution is (approximately) contained in a coordinate subspace. 
Here we present  two experiments  aimed at empirically assessing the role
of the smoothness enforcing penalty $\nor{\cdot}^2_{\hh}$ and of the parameter $\taumu$.
We first present an experiment where the support of the marginal distribution
approximately coincides with a coordinate subspace $x^2=0$. 
Then we systematically investigate the stabilizing effect of the smoothness enforcing penalty
also when the marginal distribution is not degenerate.

\paragraph{Adaption to the Marginal Distribution}
We  consider a toy problem in 2 dimensions, where the support of the marginal distribution $\marg(x^1,x^2)$
approximately coincides with the coordinate subspace $x^2=0$. 
Precisely $x^1$ is uniformly sampled from $[-1,1]$, whereas $x^2$ is drawn from a normal distribution $x^2\sim\mathcal{N}(0,0.05)$.
The output labels are drawn from
 $y = (x^1)^2 + w$, where $w$ is  a white noise, sampled from a normal distribution with zero mean and 
variance $0.1$.
Given a training set of $\n=20$ samples i.i.d. drawn from the above distribution (Figure \ref{fig:relevance_toy} top-left), 
we evaluate the optimal value of the regularization parameter $\pone$ 
via hold out validation on an independent validation set of $\n_{\textrm{val}}= \n = 20$ samples. 
We repeat the process for $\taumu=0$ and $\taumu=10$.
In both cases the reconstruction accuracy on the support of $\marg$ is high, see Figure \ref{fig:relevance_toy} bottom-right . 
However, while  $\taumu=10$ our method correctly selects the only relevant variable $x^1$ (Figure \ref{fig:relevance_toy} bottom-left),
when $\taumu=0$ both variables are selected (Figure \ref{fig:relevance_toy} bottom-center),
since functional $\emp^{\pone,0}$ is insensible to errors out of $\text{supp}(\marg)$, 
and the regularization term $\pone\Regn$ alone does not penalizes variations out of $\text{supp}(\marg)$ .
\begin{figure}[t]
\begin{center}
\caption{\label{fig:relevance_toy} 
Effect of the combined regularization $\Regn(\cdot) +\taumu\norh{\cdot}^2$ on a toy problem in $\rone^2$
where  the support of marginal distribution
approximately coincides with the coordinate subspace $x^2=0$.
The output labels are drawn from  $y= (x^1)^2 + w$, with $w\sim\mathcal{N}(0,0.1)$.
}
\includegraphics[width = .8\linewidth]{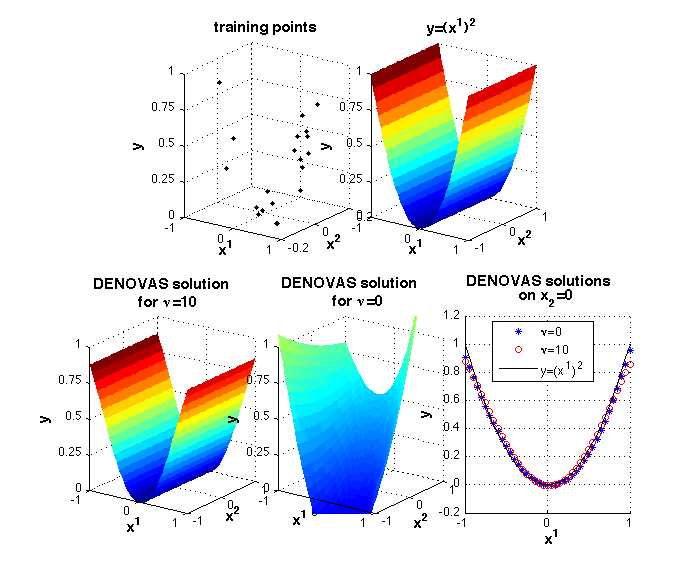}
\end{center}
\end{figure}

\paragraph{Effect of varying $\taumu$}
Here we empirically investigate the stabilizing effect of the smoothness enforcing penalty
when the marginal distribution is not degenerate.
The input variables $x = (x^1,\dots,x^{20})$ are uniformly drawn from $[-1,1]^{20}$.
The output labels are i.i.d. drawn from  $y = \lambda\sum_{\jj=1}^4\sum_{\jjj=\jj+1}^4 x^\jj x^\jjj + w$, 
where $w\sim\mathcal{N}(0,1)$, and $\lambda$ is a rescaling factor that determines the signal to noise ratio to be 15:1. 
We extract training sets of size $\n$ which varies from 40 to 120 with steps of 10.
We then apply our method  with polynomial kernel of degree $p=4$,
letting vary $\taumu$ in $\{0.1, 0.2, 0.5, 1, 2, 5, 10, 20\}$.
 For fixed $\n$ and $\taumu$ we evaluate the optimal value of the regularization parameter 
$\pone$ via hold out validation on an independent validation set of $\n_{\textrm{val}}= \n$ samples. 
We measure the  selection error as the mean of the false negative rate (fraction of relevant variables that were not selected) 
and false positive rate (fraction of irrelevant variables that were selected).
Then, we evaluate the prediction error as  the root mean square error (RMSE) error of the selected model 
on an independent test set of $\n_{\textrm{test}}=500$ samples.   
Finally we average over 50 repetitions.

In Figure \ref{fig:vario_mu} we display the prediction error,   selection error, and computing time,  versus $\n$ for different values of $\taumu$.
Clearly, if $\taumu$ is too small, both prediction and selection are poor.
For $\taumu\geq 1$ the algorithm is quite stable with respect to small variations
of $\taumu$. 
However, excessive increase of the smoothness parameter leads to a decrease 
in prediction and selection performance.  
In terms of computing time, the higher the smoothness parameter  the better the performance.
\begin{figure}[!h]
	\caption{Selection error (left), prediction error (center),  and computing time (right) 
	versus $\n$ for different values of $\taumu$. 
	The points 
	correspond to the mean 
	over the repetitions. 
	The dotted line represents the white noise standard deviation. 
	In the left figure the curves corresponding to $\taumu=5$, $\taumu=10$, and $\taumu=20$ are overlapping.}
	\label{fig:vario_mu}
	 \begin{center}
	  \includegraphics[width=\linewidth]{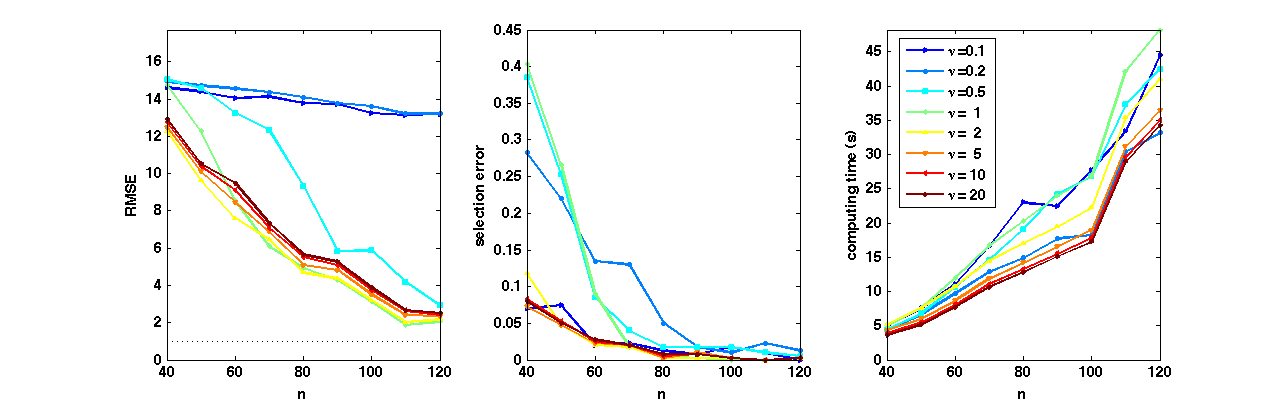}
	  \end{center}
\end{figure}

\subsubsection{Varying the model's parameters}

We present 3 sets of experiments where we evaluated the performance of our method (DENOVAS)
when varying part of the inputs parameters and leaving the others unchanged.
The parameters we take into account are the following
\begin{itemize} 
\item$\n$, training set size
\item$\dd$, input space dimensionality
\item$\deff$,   number of relevant variables
\item$p$, size of the hypotheses space, measured as the degree of the polynomial kernel.  
\end{itemize}

In all the following experiments the input variables $x = (x^1,\dots,x^d)$ are uniformly drawn from $[-1,1]^\dd$.
The output labels are computed using a noise-corrupted regression function $f$ that depends nonlinearly from a set of the input variables, 
i.e.  $y = \lambda f(x) + w$, where $w$ is  a white noise, sampled from a normal distribution with zero mean and 
variance 1, and $\lambda$ is a rescaling factor that determines the signal to noise ratio. 
 For fixed $\n, \dd,$ and $\deff$ we evaluate the optimal value of the regularization parameter 
$\pone$ via hold out validation on an independent validation set of $\n_{\textrm{val}}= \n$ samples. 

\paragraph{Varying $\n,\dd$, and $\deff$}
In this experiment we want to empirically evaluate the effect of the input space dimensionality, 
$\dd$, and the number of relevant variables, $\deff$, when the other parameters are left unchanged.
In particular we use $\dd = 10, 20, 30, 40$ and $\deff$ $= 2,3,4,5,6$. 
For each value of $\deff$ we use a different regression function,  
$f(x) = \lambda\sum_{\jj=1}^{\deff}\sum_{\jjj=\jj+1}^{\deff} c_{\jj\jjj} x^\jj x^\jjj$, so that for fixed $\deff$ all 2-way interaction terms are present, 
and the polynomial degree of the regression function is always 2.
The coefficients $c_{\jj\jjj}$ are randomly drawn from $[.5,1]$
And $\lambda$ is determined in order to set the signal to noise ratio as 15:1.
We then apply our method with polynomial kernel of degree $p=2$. 
The  regression function thus always belongs to the hypotheses space.\\
In Figure \ref{fig:vario_d_deff}, we display the selection error, and the prediction error, respectively,  
versus $\n$ for different values of $\dd$ and number of relevant variables  $\deff$.
Both errors decrease with $\n$ and  increase with $\dd$ and $\deff$. 
In order to better visualize the dependance of the selection performance with respect to $\dd$ and $\deff$, 
 in Figure \ref{fig:NMINvario_d_deff} we plotted
the minimum number of input points that are necessary in order to achieve $10\%$ of average selection error. 
It is clear by visual inspection  that $\deff$ has a higher influence  than $\dd$ on the selection performance of our method.
	\begin{figure}[!t]
	\caption{ Prediction error (top) and selection error (bottom) versus $\n$ 
	for different values of $\dd$ and number of relevant variables ($\deff$).
	The points correspond to the means over the repetitions.
	The dotted line represents the white noise standard deviation.}
	\label{fig:vario_d_deff}
	 \begin{center}
	  \includegraphics[width=1\linewidth]{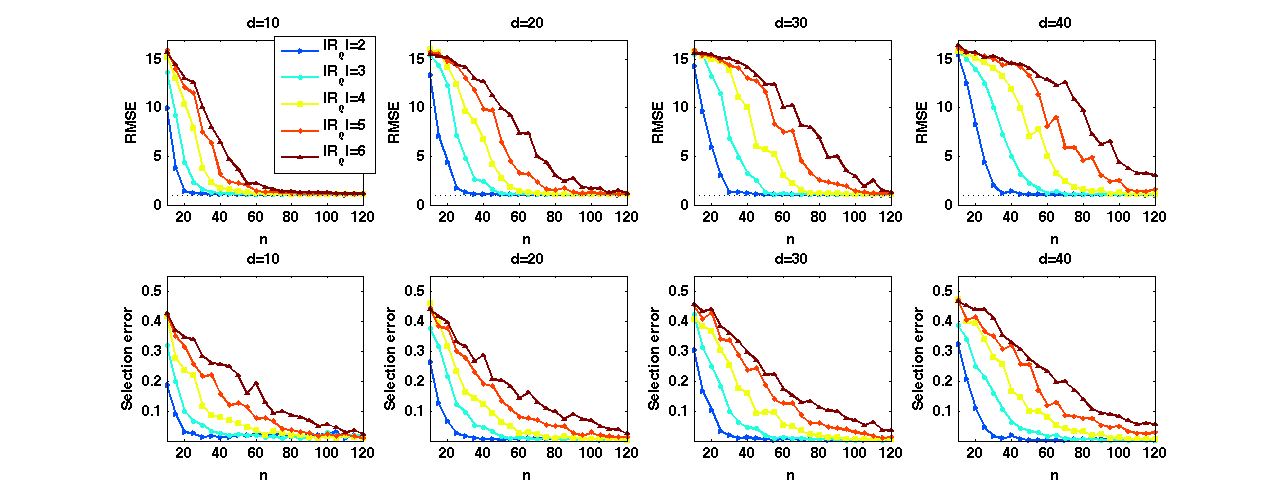}
	  \end{center}
	\end{figure}
	\begin{figure}[!t]
	\caption{ Minimum number of input points ($\n$) necessary to achieve $10\%$ of average selection error 
	 versus the number of relevant variables $\deff$ for different values of $\dd$ (left), and versus $\dd$ for different values of $\deff$ (right).}
	\label{fig:NMINvario_d_deff}
	 \begin{center}
	  \includegraphics[width=.45\linewidth]{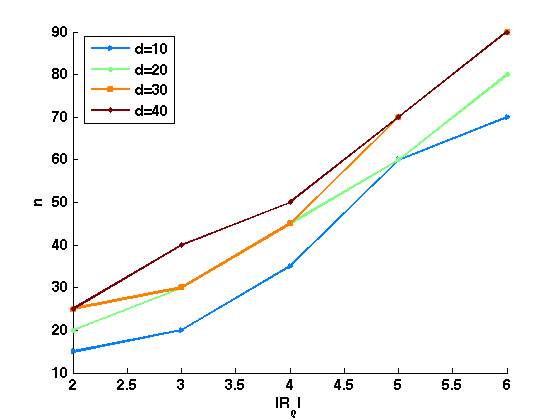}
	  \includegraphics[width=.45\linewidth]{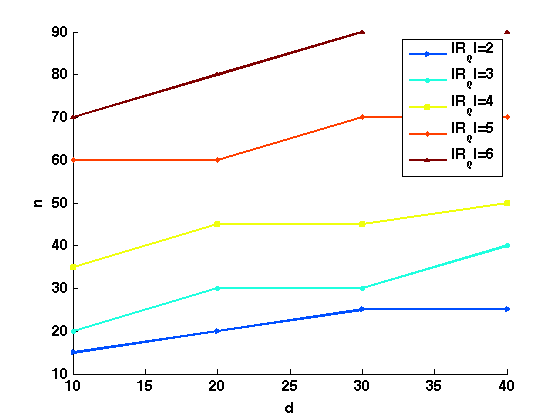}
	  \end{center}
	\end{figure}
	
\paragraph{Varying  $\n$ and $p$}
In this experiment we want to empirically evaluate the effect of the size of the hypotheses space on the performance of our method. We therefore leave unchanged the data generation setting, made exception for the number of training samples, and vary the polynomial kernel degree as $p$ $=1,2,3,4,5,6$.
We let $\dd =  20$, $\I = \{1,2\}$, and  $f(x) = x^1x^2$, and let vary $\n$ from 20 to 80 with steps of 5.
The signal to noise ratio is  3:1.\\
In Figure \ref{fig:vario_p}, we display  the prediction and  selection error,  versus $\n$, for different values of $p$.
For $p\geq2$, that is when the hypotheses space contains the regression function, 
both errors decrease with $\n$ and increase with $p$. 
Nevertheless the effect of $p$ decreases for large $p$, in fact for $p=4,5,$ and $6$, the performance is almost the same.
On the other hand, when the hypotheses space is too small to include the regression function, as for the set of linear functions ($p=1$), the selection error coincides with chance (0.5), and the prediction error is very high, even for large numbers of samples.
	\begin{figure}[!t]
	\caption{ Prediction error (left) and selection error (right) versus $\n$ 
	for different values of the polynomial kernel degree ($p$).
	The points correspond to the means over the repetitions.
	The dotted line represents the white noise standard deviation.}
	\label{fig:vario_p}
	 \begin{center}
	  \includegraphics[width=\linewidth]{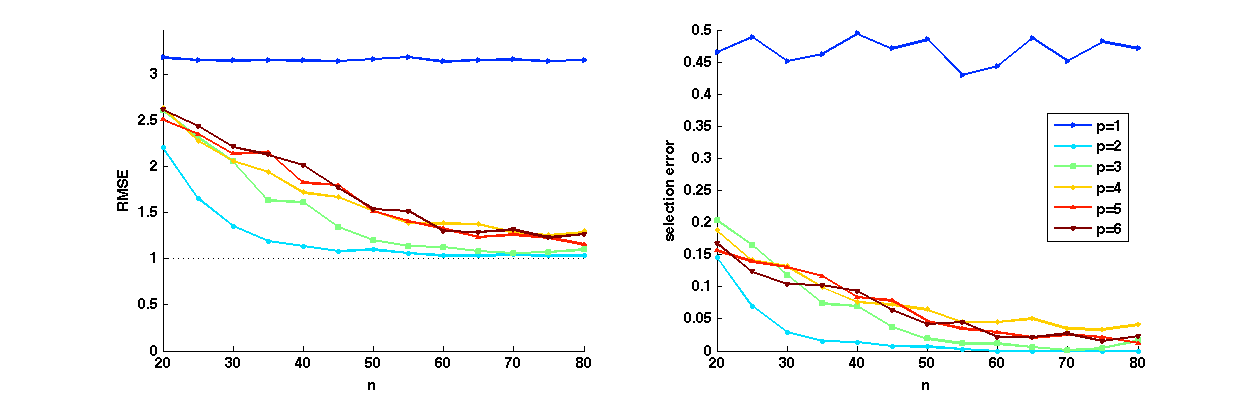}
	  \end{center}
	\end{figure}

\paragraph{Varying  the number of relevant features, for fixed $\deff$: 
comparison with $\ell^1$ regularization on the feature space}
In this experiment we want to empirically evaluate the effect of the number of features involved in the regression function (  that is the number of monomials constituting the polynomial) on the performance of our method  when $\deff$ remains the same as well as  all other input parameters. 
Note that while $\deff$ is the number of relevant variables, here we vary the number of relevant features (not variables!), 
which, in theory, has nothing to do with $\deff$. 
Furthermore we compare the performance of our method  to that of $\ell^1$ regularization on the feature space ($\ell^1$-features).
We therefore leave unchanged the data generation setting, made exception for the regression function.
We set $\dd =  10$, $\I = \{1,2,3\}$, $n=30$, and then use a polynomial kernel of degree 2. 
The signal to noise ratio is this time 3:1.
Note that with this setting the size of the features space is 66.
For fixed number of relevant features the regression function is set to be a randomly chosen linear combination of the features involving one or two of the first three variables ($x^1, (x^1)^2, x^1x^2, x^1x^3$, etc.), with the constraint that the combination must be a polynomial of degree 2, involving all 3 variables.\\
In Figure \ref{fig:vario_nfeat}, we display  the prediction  and selection error,  versus the number of relevant features.
While the performance of $\ell^1$-features fades when the number of relevant features increases,  our method  presents stable performance both in terms of selection and prediction error. 
From our simulation it appears that, while  our method   depends on the number of relevant variables, it is indeed independent of the number of features.
	\begin{figure}[!t]
	\caption{Prediction error (left) selection error (right)  versus the number of relevant features.
	The points correspond to the means over the repetitions.
	The dotted line represents the white noise standard deviation.}
	\label{fig:vario_nfeat}
	 \begin{center}
	  \includegraphics[width=\linewidth]{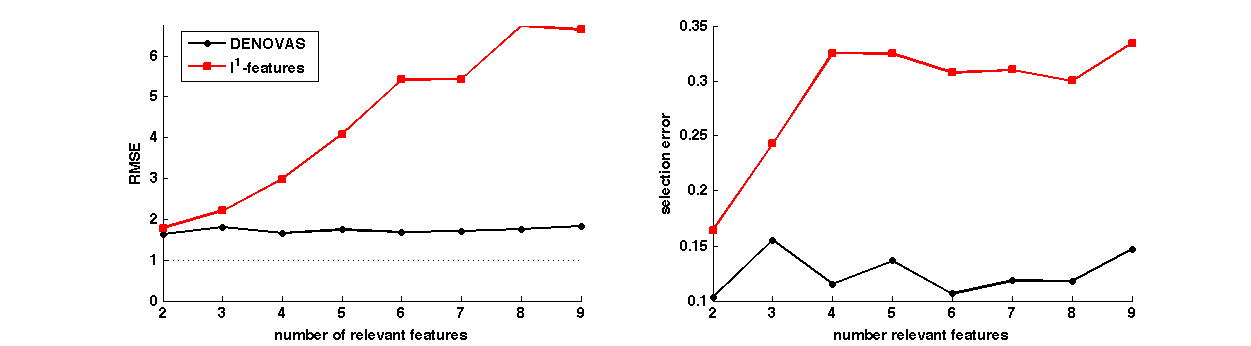}
	  \end{center}
	\end{figure}

\subsection{Comparison with Other Methods}\label{sec:comparisons}
In this section we present numerical experiments aimed at comparing our method 
 with state-of-the-art algorithms. 
In particular, since our method  is a regularization method, we focus on alternative regularization approaches to the problem of nonlinear variable selection.
For comparisons with more general techniques for nonlinear variable selection  we refer the interested reader to \cite{bach09}. 
\subsubsection{Compared algorithms}\label{sec:algos}

We consider the following regularization algorithms:
\begin{itemize}
\item Additive models with multiple kernels, that is  Multiple Kernel Learning (MKL)
\item $\ell^1$ regularization on the feature space associated to a polynomial kernel ($\ell^1$-features)
\item COSSO \cite{lin2006} with 1-way interaction (COSSO1) and 2-way interactions (COSSO2)
\footnote{In all toy data, and in part of the real data, the following warning message was displayed:\\
\texttt{Maximum number of iterations exceeded;
increase options.MaxIter. \\
To continue solving the problem with the current solution as the starting point, \\
set x0 = x before calling lsqlin}.\\
In those cases the algorithm did not reach convergence in a reasonable amount of time, 
therefore  the error bars corresponding to COSSO2 were omitted.}
\item  Hierarchical Kernel Learning \cite{bach09} with polynomial (HKL pol.) and hermite (HKL herm.) kernel
\item  Regularized Least Squares (RLS).
\end{itemize}
Note that, differently from the first 4 methods,  RLS is not a variable selection algorithm, however we consider it since it is typically a good benchmark for the prediction error. 

For $\ell^1$-features and MKL we use our own Matlab implementation based on proximal methods (for details see \cite{mosci2010ecml}).
For COSSO we used the Matlab code available at \url{www.stat.wisc.edu/~yilin} or \url{www4.stat.ncsu.edu/~hzhang} which can deal with 1 and 2-way interactions. 
For HKL we used the code available online at \url{http://www.di.ens.fr/~fbach/hkl/index.html}.
While for MKL and $\ell^1$-features we are able to identify the set of selected variables, 
for COSSO and HKL  extracting the sparsity patterns from the available code is not straightforward.
We therefore compute the  selection errors only for $\ell^1$-features, MKL, and our method . 

\subsubsection{Synthetic data}

We simulated data with $\dd$ input variables, where each variable is uniformly sampled from [-2,2]. The output $y$ is a nonlinear function  of  the first 4 variables, $y = f(x^1, x^2,x^3,x^4)+\epsilon$, 
where epsilon is a white noise, $\epsilon \sim \mathcal{N}(0,\sigma^2)$, and $\sigma$ is chosen so that the signal to noise ratio is 15:1.
We consider the 4 models described in Table \ref{tab:toy_design}.
\begin{table}[h]
\caption{Synthetic data design}
\label{tab:toy_design}
\begin{center}
\begin{tabular}{lccccc}
	&~~~~{$\mathbf{\dd}$}~~~~&\begin{tabular}{cc}{\bf number of }\\ {\bf relevant variables}\end{tabular}&~~~~{$\mathbf{\n}$}~~~~&{\bf model ($f$)} \\
\hline
additive p=2&$40$ 	&$4$  &$100$  &$y=\sum_{\jj=1}^4 (x^\jj)^2$\\
2way p=2	&$40$  	&$4$  &$100$  &$y=\sum_{\jj=1}^4\sum_{\jjj=\jj+1}^4x^\jj x^\jjj$\\
3way p=6	&$40$  	&$3$  &$100$  &$y= (x^1x^2x^3)^2$\\
radial	&$20$  	&$2$  &$100$  &$y=\frac{1}{\pi}((x^1)^2+(x^2)^2)e^{-((x^1)^2+(x^2)^2)}$ \\
\end{tabular}
\end{center}
\end{table}

For model selection and testing we follow the same protocol described at the beginning of  Section \ref{sec:emp_val}, with $\n=100,100$ and $1000$ for training, validation and testing, respectively.  Finally we average over 20 repetitions.
In the first 3 models, for MKL, HKL, RLS and our method  we employed the polynomial kernel of degree $p$, where $p$ is the polynomial degree of the regression function $f$.
For $\ell^1$-features  we used the polynomial kernel with degree chosen as the minimum between the polynomial degree of $f$ and 3. This was due to computational reasons, 
in fact,  with $p=4$  and $\dd=40$, the number of features is highly above $100,000$. 
For the last model, we used the polynomial kernel of degree  $4$ for MKL, $\ell^1$-features and HKL, 
and the Gaussian kernel with kernel parameter $\sigma=2$ for RLS and our method \footnote{Note that here we are interested in evaluating the ability of our method  of dealing with a general kernel like the Gaussian kernel, not in the choice of the kernel parameter itself.
Nonetheless, a data driven choice for $\sigma$ will be presented in the real data experiments in Subsection \ref{sec:real_data}.}.
COSSO2 never reached convergence. 
Results in terms of prediction and selection errors are reported  in Figure \ref{fig:toy_pred_err}. 

\begin{figure}[h!]
	  \caption{Prediction error (top) and fraction of selected variables (bottom) on synthetic data for 
	  the proposed method (DENOVAS), multiple kernel learning for additive models (MKL), 
	  $\ell^1$ regularization on the feature space associated to a polynomial kernel ($\ell^1$-features),
	  COSSO with 1-way interactions (COSSO1), hierarchical kernel learning with polynomial kernel (HKL pol.),
	  and   regularized least squares (RLS).
	   The dotted line in the upper plot corresponds  to the white noise standard deviation.
	   Selection errors for COSSO, and HKL are not reported because they are not straightforwardly computable from the available code.
	   }
\label{fig:toy_pred_err}
	 \begin{center}
		  \includegraphics[width=1\linewidth]{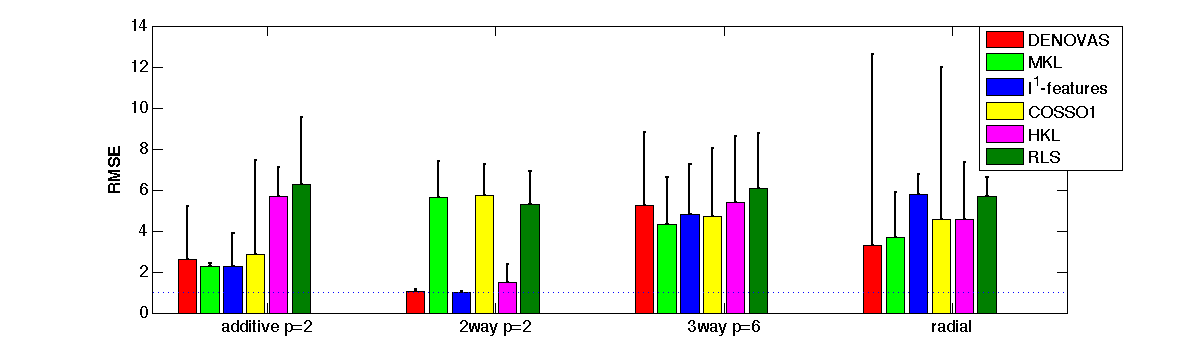}
%
		  \includegraphics[width=1\linewidth]{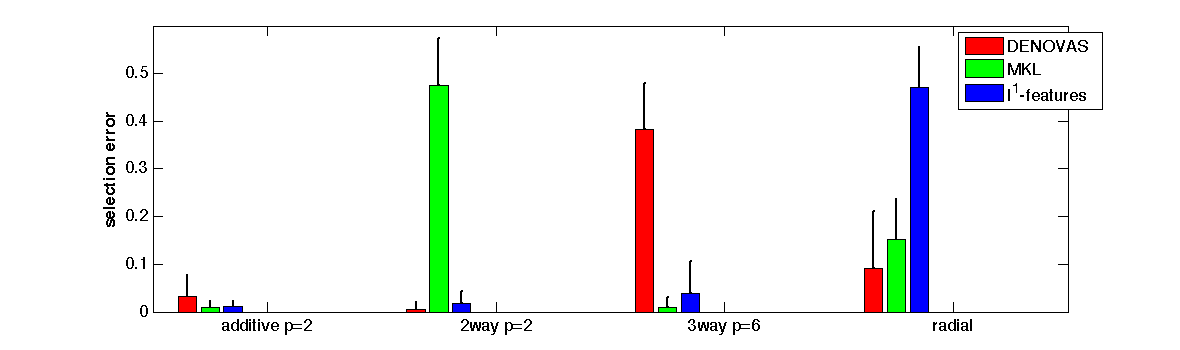}
	  \end{center}
\end{figure}

When the regression function is simple (low interaction degree or low 
polynomial degree) more tailored algorithms, such as MKL--which is additive by design-- for the experiment ``additive p=2'', or $\ell^1$-features for experiments ``2way p=4'' -- in this case the dictionary size is less than 1000--,  compare favorably with respect to our method.
However, when the nonlinearity of the regression function favors the use of a large hypotheses space,
our method  significantly outperforms the other methods. 
This is particularly evident in the experiment ``radial'', which was anticipated  in Section \ref{sec:approach},
where we plotted in Figure \ref{fig:radial_intro} the regression function and its estimates obtained with the different algorithms 
for one of the 20 repetitions.

\subsubsection{Real data}\label{sec:real_data}
We consider the 7 benchmark data sets described in Table \ref{tab:real_data}.
\begin{table}[h!]
\caption{Real data sets}
\begin{center}
\label{tab:real_data}
\begin{tabular}{lcccc}
 &  {\bf number of} & {\bf number of} & &\\
{\bf data name} &  {\bf input variables} & {\bf instances} & {\bf source} & {\bf task}\\
\hline\\
boston housing & 13 & 506 & LIACC\footnote{\url{http://www.liaad.up.pt/~ltorgo/Regression/DataSets.html}} &regression\\
census & 16 & 22784 & LIACC &regression\\
delta ailerons & 5 & 7129 & LIACC&regression\\
stock & 10 & 950 &LIACC &regression\\
image segmentation & 18 & 2310 &  IDA\footnote{IDA benchmark  repository (\url{http://www.fml.tuebingen.mpg.de/Members/raetsch/benchmark})} &classification\\  
pole telecomm & 26\footnote{we removed 12 constant variables} & 15000 &LIACC &regression\\
breast cancer &  32 & 198 & UCI  \footnote{machine learning repository (\url{http://archive.ics.uci.edu/ml/datasets.html})} & regression
\end{tabular}
\end{center}
\end{table}
We build training and validation sets by randomly drawing $n_{\textrm{train}}$ and $n_{\textrm{val}}$ samples, and using the remaining samples for testing. 
For the first 6 data sets we let $n_{\textrm{train}}=n_{\textrm{val}} = 150$, 
whereas for breast cancer data we let $n_{\textrm{train}}=n_{\textrm{val}} =60$.
We then apply the  algorithms described in Subsection \ref{sec:algos}.
with the validation protocol described in  Section \ref{sec:emp_val}.
For our method  and RLS we used the gaussian kernel with the kernel parameter $\sigma$ chosen as the mean over the samples of the euclidean distance form the 20-th nearest neighbor. 
Since the other methods cannot be run with the gaussian kernel we used a polynomial kernel of degree $p=3$ for MKL and $\ell^1$-features.
For HKL we used both the polynomial kernel and the hermite kernel, both with $p=3$.
Results in terms of prediction and selection error are reported in Figure \ref{fig:real_data}.\\
Some of the data, such as the stock data, seem not to be variable selection problem, in fact the best performance is achieved by our method  though selecting (almost) all variables, or, equivalently by RLS. Our method  outperforms all other methods on several data sets. 
In most cases, the performance of our method  and RLS are similar. 
Nonetheless our method  brings higher interpretability since it is able to select a smaller subset of relevant variable, while the estimate provided by RLS depends on all variables. 


\begin{figure}
	 \begin{center}
		  \includegraphics[width=1\linewidth]{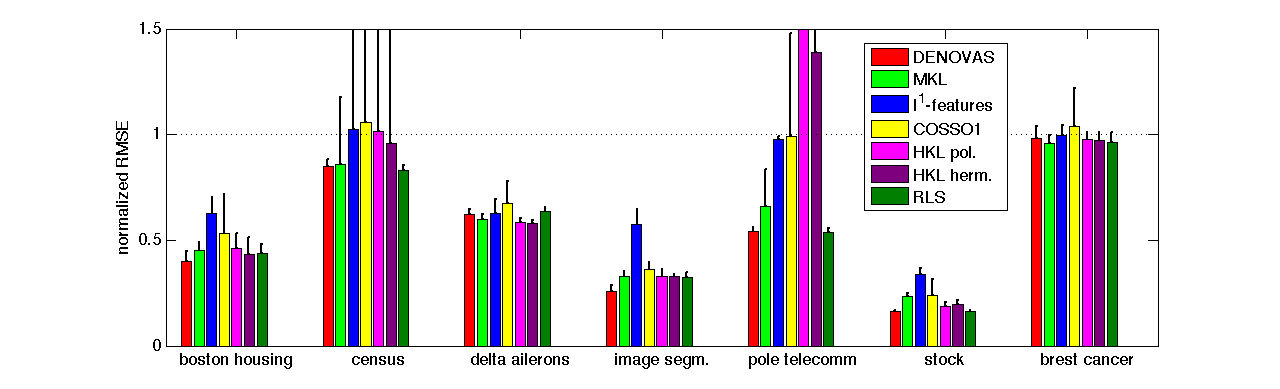}
		  \includegraphics[width=1\linewidth]{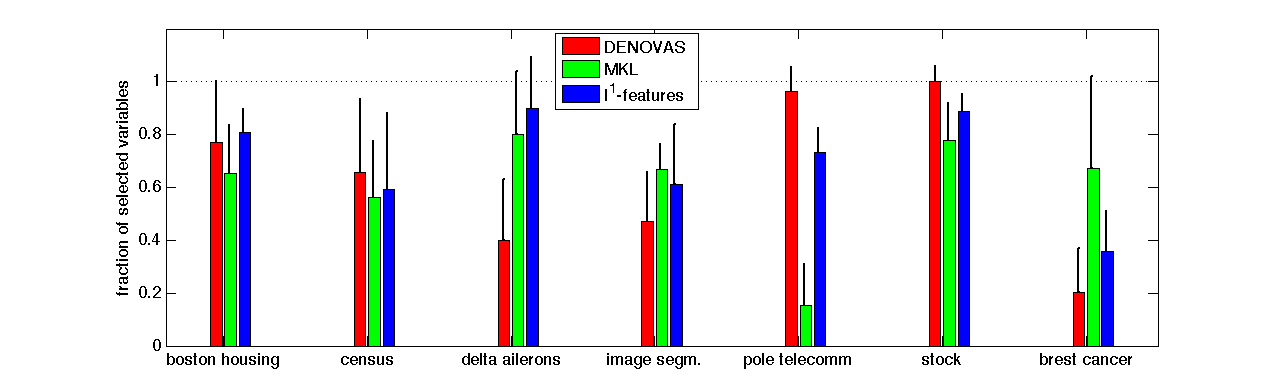}
  		  \includegraphics[width=1\linewidth]{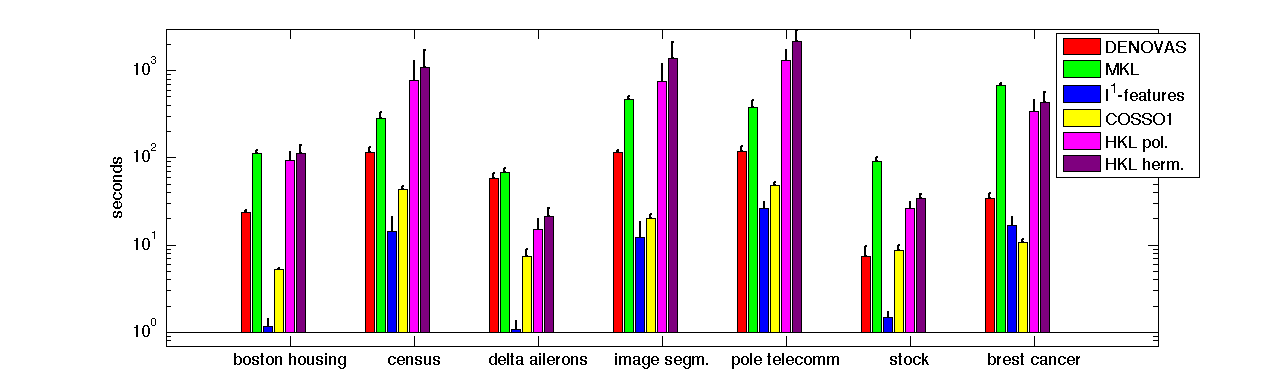}
	  \end{center}
	  \vspace{-0.7cm}
	  \caption{\small{Prediction error (top) and fraction of selected variables (center) and computing time (bottom) on real data for 
	  the proposed method (DENOVAS), multiple kernel learning for univariate additive functions (MKL),  
	  $\ell^1$ regularization on the feature space associated to a polynomial kernel ($\ell^1$-features),
	  COSSO with 1-way interactions (COSSO1), hierarchical kernel learning with polynomial kernel (HKL pol.),
	   hierarchical kernel learning with hermite kernel (HKL herm.) and   regularized least squares (RLS).
	  Prediction errors for  COSSO2 are not reported because it is always outperformed by COSSO1.
	  such errors were still too large to report in the first three data sets, 
	  and were not available since the algorithm did not reach convergence for image segmentation, pole telecomm and breast cancer data. 
	  To make the prediction errors comparable among experiments, 
	  root mean squared errors (RMSE)  were divided by the outputs  standard deviation, which corresponds to the dotted line. 
	  Error bars are the standard deviations of the normalized RMSE.
	  Though the largest normalized RMSE appear out of the figure axis,
	   we preferred to display the prediction errors with the current axes limits 
	   in order to allow the reader to appreciate the difference between the smallest, and thus most significant, errors.
	   Selection errors for COSSO, and HKL are not reported because they are not straightforwardly computable from the available code.
	   The computing time corresponds to the entire model selection and testing protocol. 
	   Computing time for RLS is not reported because it was always negligible with respect to the other methods.}
	  }\label{fig:real_data}
\end{figure}


We also run experiments on the same 7 data sets with different kernel choices for our method . We consider the polynomial kernel with degree $p=2, 3$ and $4$, and the gaussian kernel. Comparisons among the different kernels in terms of prediction and selection accuracy are plotted in Figure \ref{fig:real_dataDENOVAS}. Interestingly the choice of the gaussian kernel seems to be the preferable choice in most data sets.

\begin{figure}[h]
	 \begin{center}
		  \includegraphics[width=1\linewidth]{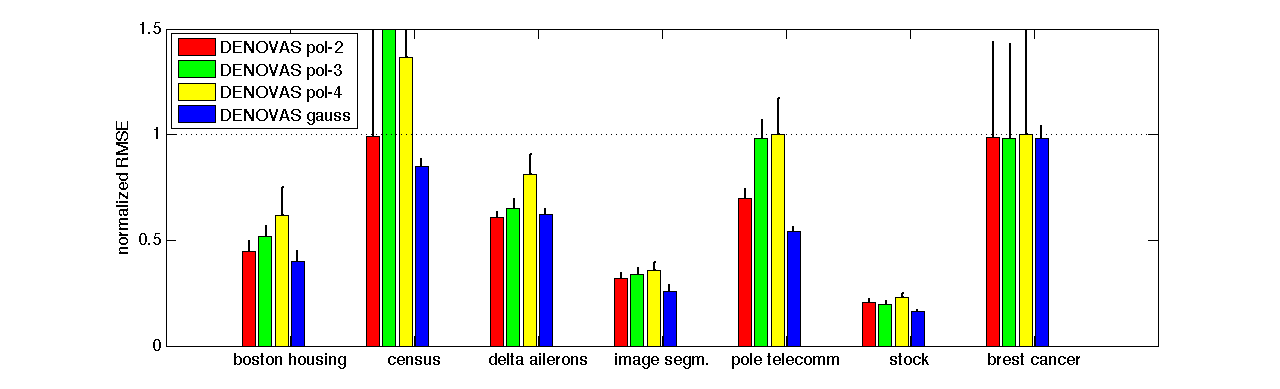}
		  \includegraphics[width=1\linewidth]{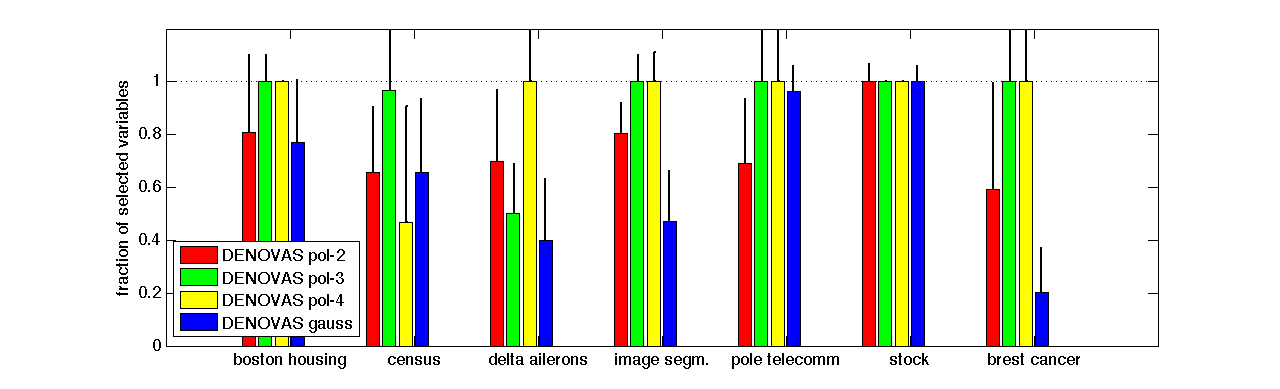}
	  \end{center}
	  \caption{Prediction error (top) and fraction of selected variables (bottom) on real data for our method  with different kernels:
	  polynomial kernel of degree $p=1, 2$ and 3 (DENOVAS pol-$p$), and Guassian kernel (DENOVAS gauss). 
	  Error bars represent the standard deviations.
	   In order to make the prediction errors comparable among experiments, root mean square errors  were divided by the outputs  standard deviation, which corresponds to the dotted line.}\label{fig:real_dataDENOVAS}
\end{figure}

 \section{Discussion}\label{sec:disc}

Sparsity based method has recently emerged as way to perform learning and variable selection
from high dimensional data. So far, compared to other machine learning techniques, 
this class of methods suffers from strong modeling assumptions and is in fact limited to 
parametric or semi-parametric models (additive models).
In this paper we discuss a possible way to circumvent this shortcoming and 
exploit sparsity ideas in a non-parametric context. 

We propose to use partial derivatives of functions in a RKHS to design 
a new sparsity penalty and a corresponding regularization scheme.
Using  results from the theory of RKHS and proximal methods
we show that the regularized estimator can be provably computed 
through an iterative procedure.
The consistency property of the proposed estimator are studied. 
 Exploiting the non-parametric nature of the method we can prove universal consistency.
 Moreover we study  selection properties and show that 
 that 
the proposed regularization scheme represents a safe
filter for variable selection, as it does not discard relevant variables.
Extensive simulations  on synthetic data demonstrate
the prediction and selection properties of the proposed algorithm.
Finally,  comparisons to state-of-the-art algorithms for nonlinear variable selection
on toy data as well as on a cohort of benchmark data sets,   
show that our approach often leads to better prediction and selection performance.

Our work can be considered as a first step towards understanding the role of 
sparsity beyond additive models. It substantially differs with respect to previous approaches  based
on local polinomial regression \cite{lafferty08,bertin08,MilHal10}, since it is a regularization scheme 
directly performing global variable selection. The RKHSs' machinery allows on the one hand to find a computationally efficient
algorithmic solution, and on the other hand to consider very general probability distributions $\rho$, which are not required to
have a positive density with respect to the Lebesgue measure (differently from \cite{comminges2011}).
Several research directions are yet to be explored.
\begin{itemize}
\item From a theoretical point of view it would be interesting to further analyzing the sparsity property of the obtained estimator in terms of finite sample estimates
for the prediction and the selection error. 
\item From a computational point of view the main question is whether our method can be scaled to work in very high dimensions.
Current computations are limited  by memory constraints. A variety of  method for large scale optimization can be considered towards this end.
\item A natural by product of computational improvements would be the possibility of considering a semi-supervised setting which is  naturally suggested by our approach. More generally we plan to investigate the application of the  RKHS representation for differential operators in unsupervised learning. 
\item More generally,  our study begs the question of whether there are alternative/better ways to
perform learning and variable selection beyond additive models and using non parametric models. 
\end{itemize}

\small{ {\bf Acknowldgments} LR is assistant professor at DIBRIS, Universita` di Genova, Italy and currently on leave of absence. 
The authors would like to thank Ernesto De Vito for many useful discussions 
and suggesting  the proof of Lemma 4. SM and LR would like to thank 
Francis Bach and Guillame Obozinski for useful discussions.
This paper describes a joint research work done at 
and at the Departments of Computer Science and
Mathematics of the University of Genoa and at the IIT@MIT lab hosted in 
 the Center for Biological and Computational Learning (within the McGovern Institute for
Brain Research at MIT), at the Department of Brain and Cognitive
Sciences (affiliated with the Computer Sciences and Artificial
Intelligence Laboratory), The authors
have been partially supported by the Integrated Project Health-e-Child
IST-2004-027749 and by grants from DARPA (IPTO and DSO), National
Science Foundation (NSF-0640097, NSF-0827427), and  Compagnia di San
Paolo, Torino. Additional support was provided by: Adobe, Honda Research
Institute USA, King  Abdullah University Science and Technology grant to
B. DeVore, NEC, Sony and especially by the Eugene McDermott Foundation.  }

\appendix \label{app:theorem}



\section{Derivatives in RKHS and Representer Theorem}\label{tools_rkhs}
 
 Consider  $\ldue=\{f:\X\to\rone~ \text{measurable}|~\int |f(x)|^2d\marg(x)<\infty\}$
 and  $\rone^\n$ with inner product normalized by a factor $1/\n$, $\nor{\cdot}_\n$. 

The operator $\Ik:\hh\to\ldue$ defined by  $(\Ik f)(x)=\scalh{f}{\kk_x}$,  for almost all $x\in X$, 
is well-defined and bounded thanks to  assumption A1.  The sampling operator \eqref{samp_op} 
can be seen as its empirical counterpart. 
Similarly $\der_\jj:\hh\to \ldue$ defined by $(\der_\jj f)(x)=\scal{f}{(\derh_\jj k)_{x}}$,  for almost all $x\in X$ and $\jj=1, \dots \dd$, 
is well-defined and bounded thanks to  assumption A2.  The  operator \eqref{emp_der} 
can be seen as its empirical counterpart. 
Several properties of such operators and related quantities are given by the following two propositions.
\ \\
\begin{proposition}\label{prop:HS}
If assumptions A1 and A2 are met, the  operator $\Ik$ and the continuous partial derivative $\der_\jj$ are
Hilbert-Schmidt operators from $\hh$ to $\ldue$, and
\begin{align*}
 \Ik^* g(t)=\int_\X \kk_{x}(t) g(x)d\marg(x),\qquad\qquad&
 \Ik^*\Ik f(t)=\int_\X \kk_{x} (t)\scalh{f}{k_x}d\marg(x)\\
 \der^*_\jj g(t)=\int_\X (\derh_\jj\kk)_{x}(t) g(x)d\marg(x),\qquad\qquad&
\der^*_\jj\der_\jjj f(t)=\int_\X (\derh_\jj\kk)_{x}(t) \scalh{f}{(\derh_\jjj\kk)_x}d\marg(x)\\
\end{align*}
\end{proposition}
\ \\
\begin{proposition} \label{prop:HSemp}
If assumptions A1 and A2 are met, the sampling operator $\Sx$ and the empirical partial derivative $\dern_\jj$ are
Hilbert-Schmidt operators from $\hh$ to $\rone^\n$, and
\begin{align*}
 \Sx^* v=\frac{1}{\n} \sum_{\ii=1}^\n\kk_{x_\ii} v_\ii,\qquad\qquad&
 \Sx^*\Sx f=\frac{1}{\n} \sum_{\ii=1}^\n\kk_{x_\ii} \scalh{f}{k_{x_\ii}}\\
 \dern^*_\jj v=\frac{1}{\n}\sum_{\ii=1}^\n (\derh_\jj\kk)_{x_\ii} v_\ii,\qquad\qquad&
\dern^*_\jj\dern_\jjj f=\frac{1}{\n} \sum_{\ii=1}^\n(\derh_\jj\kk)_{x_\ii} \scalh{f}{(\derh_\jjj\kk)_{x_\ii}}\\
\end{align*}
where $\jj,\jjj=1, \dots, \dd$.
\end{proposition}

The proof can be found in \cite{devito05} for $\Ik$ and $\Sx$, where assumption A1 is used. 
The proof for $\der_\jj$ and $\dern_\jj$ is based on the same tools and on assumption A2.
Furthermore, a similar result can be obtained for the continuous and empirical gradient
\begin{align*}
&\grad:\hh\to (\ldue)^\dd,\quad &\grad f &= (\der_\jj f)_{\jj=1}^\dd\\
&\gradn:\hh\to (\rone^\n)^\dd,\quad &\gradn f&= (\dern_\jj f)_{\jj=1}^\dd,
\end{align*}
which can be shown to be Hilbert-Schmidt operators from $\hh$ to $(\ldue)^\dd$ and from $\hh$ to $(\rone^\n)^\dd$, respectively.\\

We next restate Proposition \ref{lemma:repr_theo} in a slightly more abstract form and give its proof. \\
 {\bf Proposition [Proposition \ref{lemma:repr_theo} Extended]}
{\it  The minimizer of \eqn{algo_computation} satisfies
  $
\fz \in \text{Range}(\Sx^*)+ \text{Range}(\gradn^*).
  $
  Henceforth it satisfies  the following  representer theorem
  \beeq{func_rep}{
\fz=\Sx^* \alpha+ \gradn^* \beta=
  \sum_{\ii=1}^\n \frac 1 \n\alpha_\ii k_{x_\ii}+
  \sum_{\ii=1}^\n \sum_{\jj=1}^\dd \frac 1 \n \beta_{\jj,\ii}(\derh_\jj k)_{x_\ii}
  }
  with $\alpha \in \rone^\n$ and $\beta  \in \rone^{n\dd}$.}
    
\begin{proof}
 Being $\text{Range}(\Sx^*)+ \text{Range}(\gradn^*)$ a closed subspace of $\hh$, we can write 
 any function $f\in\hh$ as $f=f^{/\!/}+f^\bot$,
where $f^{/\!/}\in\text{Range}(\Sx^*)+ \text{Range}(\gradn^*)$
and $\scal{f^\bot}{g}_\hh$ for all $g\in\text{Range}(\Sx^*)+ \text{Range}(\gradn^*)$.
Now if we plug the decomposition $f=f^{/\!/}+f^\bot$ in the variational problem \eqref{algo_computation}, we obtain
$$
\fz = \argmin_{f\in\hh, ~f=f^{/\!/}+f^\bot} 
\left\{\emp (f^{/\!/}) +2\pone\Regn(f^{/\!/}) + \pone \taumu\nor{f^{/\!/}}^2_\hh+\pone \taumu\nor{f^\bot}^2_\hh
\right\}
$$
which is clearly minimized by $f^\bot=0$.
The second equality in \eqref{func_rep} then derives directly
from definition of $\Sx^*$  and $\gradn^*$. 
 \end{proof}
We conclude with the following example on how to compute derivatives and related quantities for the Gaussian Kernel.

\begin{example} \label{DerExample}
Note that all the terms involved in \eqref{repr_itext} are explictly computable. As an example we show how to compute them when $k(x,s)=e^{-\frac{\nor{x-s}^2}{2\gamma^2}}$ is the gaussian kernel on $\rone^\dd$. By definition
\[
(\derh_\jj k)_{x_{\ii}}(x)=\langle\left.\frac{\partial k(s,\cdot)}{\partial s^\jj}\right|_{s=x_\ii},k_x\rangle_\hh.
\]
Given $x\in \X$ it holds
\[
\frac{\partial k(s,x)}{\partial s^\jj}=e^{-\frac{\nor{x-s}^2}{2\gamma^2}}\cdot\left(-\frac{s^\jj-x^\jj}{\gamma^2}\right)
\quad\Longrightarrow\quad
(\derh_\jj k)_{x_{\ii}}(x)=e^{-\frac{\nor{x-x_\ii}^2}{2\gamma^2}}\cdot\left(-\frac{x_\ii^\jj-x^\jj}{\gamma^2}\right).
\]
Moreover, as we mentioned above, the computation of $\beta^\itext_{\jj,\ii}$ and $\alpha^\itext_\ii$ requires the knowledge of matrices  $\matK, \matZ_\jj, \matZ, \matL_\jj$. Also their entries are easily found starting from the kernel and the training points. We only show how the entries of  $\matZ$ and $\matL_\jj$ look like.
Using the previous computations we immediately get
\[
[\matZ_\jj]_{\ii,\iii}=e^{-\frac{\nor{x_\iii-x_\ii}^2}{2\gamma^2}}\cdot\left(-\frac{x_\ii^\jj-x_\iii^\jj}{\gamma^2}\right).
\]
In order to compute $\matL_\jj$ we need the second partial derivatives of the kernel:
\[
\frac{\partial k(s,x)}{\partial x^\jjj\partial s^\jj}=\begin{cases} -e^{-\frac{\nor{x-s}^2}{2\gamma^2}}\cdot\frac{s^\jj-x^\jj}{\gamma^2}\cdot \frac{s^\jjj-x^\jjj}{\gamma^2} & \text{ if } a\neq b\\
\ \\
-e^{-\frac{\nor{x-s}^2}{2\gamma^2}}\cdot \left(\frac{(s^\jj-x^\jj)^2}{\gamma^2}-\frac{1}{\gamma^2}\right) & \text{ if } a=b.
\end{cases}
\]
so that
\[
[\matL_{\jj,\jjj}]_{\ii,\iii}=\begin{cases} -e^{-\frac{\nor{x_\iii-x_\ii}^2}{2\gamma^2}}\cdot\frac{x_\ii^\jj-x_\iii^\jj}{\gamma^2}\cdot \frac{x_\ii^\jjj-x_\iii^\jjj}{\gamma^2} & \text{ if } a\neq b\\
\ \\
-e^{-\frac{\nor{x_\iii-x_\ii}^2}{2\gamma^2}}\cdot \left(\frac{(x_\ii^\jj-x_\iii^\jj)^2}{\gamma^2}-\frac{1}{\gamma^2}\right) & \text{ if } a=b.
\end{cases}
\]
\end{example}

\section{Proofs of Section \ref{sec:computation}} \label{app:computation}

In this appendix we collect the proofs related to the derivation of the iterative procedure given in Algorithm \ref{algo:DENOVAS}.
Theorem \ref{theo:convergence} is a consequence  of the general results about convergence of accelerated and inexact FB-splitting algorithms in \cite{salzo2011prox}. In that paper it is shown that inexact schemes converge only when the errors in the computation of the proximity operator are of a suitable type and satisfy a sufficiently fast decay condition. We first introduce the notion of admissible approximations.

\begin{definition}\label{def:inexactprox} Let $\eps\geq0$ and $\lambda>0$. We say that $h\in\hh$ is an \emph{approximation  of $\mathrm{prox}_{\lambda\Regn}(f)$  with $\eps$-precision} and we write $h\approxeq_\eps \mathrm{prox}_{\lambda\Regn}(f)$, if and only if 
\begin{equation}\label{eq:inexactprox}
 \frac{f-h}{\lambda}\in\partial_{\frac{\eps^2}{2\lambda}} \Regn(h), 
\end{equation}
where $\partial_{\frac{\eps^2}{2\lambda}}$ denotes the $\eps$-subdifferential.\footnote{ Recall that the $\eps$-subdifferential $\partial_\epsilon$ 
of a convex functional $\Omega:\hh\to \rone\cup\{+\infty\}$  is defined as the set
$$
\partial_\eps  \Omega (f) := \{h \in \hh\,:\, \Omega(g)-\Omega(f)\geq \scalh{h}{g-f}-\epsilon,~~ \forall g\in\hh\}, \qquad \forall f\in\hh.
$$}

\end{definition}

We will need the following results from \cite{salzo2011prox}.
\begin{theorem}\label{thm:err2} Consider the following inexact version of the accelerated FB-algorithm in \eqref{basic_prox} with $c_{1,\itext}$ and $c_{2,\itext}$ as in \eqref{eq:ct}
\beeq{inbasic_prox}{
f^\itext \approxeq_{\eps^\itext}
\text{prox}_{\frac{\pone}{\sigma} \Regn}
\Big(
\Big(I-\frac{1}{2\sigma} \nabla F\Big)(c_{1,t}f^{\itext-1}+c_{2,t}f^{\itext-2})
\Big) .
} 
Then, if $\eps^\itext\sim 1/\itext^l$ with $l>3/2$,    there exists a constant $C>0$ such that 
\[\emp^{\pone}(f^\itext)-\inf \emp^{\pone} \leq \frac{C}{\itext^2}.\]
\end{theorem}

\begin{proposition}\label{pr:svb}  Suppose that $\Omega:\hh\to \mathbb{R}\cup\{+\infty\}$ can be written as $\Omega=\omega\circ B$, where $B:\hh\to \mathcal{G}$ is a linear and bounded operator between Hilbert spaces and $\omega:\mathcal{G}\to \mathbb{R}\cup\{+\infty\}$ is a one-homogeneous function such that  $S:=\partial \omega(0)$ is bounded. Then for any $f\in\hh$ and any $v\in S$ such that 
\[
2\lambda \omega(B f) -2\langle \lambda B^*v,f\rangle \leq \varepsilon^2
\]
it holds
\[
f-\lambda B^*v\approxeq_{\eps} \mathrm{prox}_{\lambda \Omega}(f).
\]
\end{proposition} 
\begin{proof} [Proof of Theorem \ref{theo:convergence}]
Since the the regularizer $\Regn$ can be written as a composition of $\omega\circ B$, with $B=\gradn$ and $\omega:\rone^\dd\to[0,+\infty)$, $\omega(v)=\sum_{\jj=1}^\dd\nor{v_\jj}_\n$ Proposition \ref{pr:svb} applied with $\lambda=\pone/\step$, ensures that each sequence of the type $\gradn^*v^\itint$ which meets the condition  \eqref{lim-proj} generates, via \eqref{itext}, admissible approximations of $\text{prox}_{\frac{\pone}{\sigma} \Regn}$. 
Therefore, if  $\eps^\itext$ is such that $\eps^\itext\sim 1/t^l$ with $l>3/2$, Theorem \ref{thm:err2} implies that the  inexact version of the  FB-splitting algorithm in \eqref{itext}  shares the $1/\itext^2$ convergence rate. Equation \eqref{lim-f} directly follows from the definition of strong convexity,
\[
\frac{\pone\taumu}{8}\nor{f^\itext-\fz}^2\leq \emp^{\pone}(f^t)/2+\emp^{\pone}(\fz)/2-\emp^{\pone} (f^\itext/2+\fz/2)\leq \frac{1}{2}(\emp^{\pone}(f^t)-\emp^{\pone}(\fz))
\]
\end{proof}
			
\begin{proof}[Proof of Proposition \ref{prop:itext}]
We first show that the matrices $\matK, \matZ_\jj, \matL_\jj$ 
defined in \eqref{defK},\eqref{defZ}, and \eqref{defL}, 
are the matrices associated to  the  operators
$\Sx \Sx^*: \rone^\n \to \rone^\n$, $\Sx \dern_\jj^*: \rone^\n \to \rone^\n$
and $\dern_\jj\gradn^*:\rone^{\n\dd}\to \rone^\n$, respectively.
For  $\matK$, the proof is trivial and derives directly from the definition 
of adjoint of $\Sx$-- see Proposition \ref{prop:HSemp}.
For  $\matZ_\jj$ and $\matZ$, from the definition of $\dern_\jj^*$ we have that
$$
\left(\Sx\dern_\jj^* \alpha\right)_{\ii} =
\frac{1}{\n} \sum_{\iii=1}^\n \alpha_\iii (\derh_\jj k)_{x_\iii} (x_\ii) =
\sum_{\iii=1}^\n (\matZ_\jj)_{\ii,\iii}\alpha_\iii=
(\matZ_\jj \alpha)_\ii,
$$
so that
$\Sx\gradn^*\beta = \sum_{\jj=1}^\dd\Sx\dern_\jj^* (\beta_{\jj,\ii})_{\ii=1}^\n=\sum_{\jj=1}^\dd \matZ_\jj (\beta_{\jj,\ii})_{\ii=1}^\n=\matZ \beta$.
For  $\matL_\jj$, we first observe that
$$
 \scalh{(\derh_\jj\kk)_{x}}{(\derh_\jjj\kk)_{x'}}=
 \left. \frac{\partial (\derh_\jjj\kk)_{x'}(t)}{\partial t^\jj}\right|_{t=x}
=\left.\frac{\partial^2 k(s,t)}{\partial t^\jj\partial s^\jjj}\right|_{t=x,s=x'},
$$
so that operator $\dern_\jj {\dern_\jjj}^*: \rone^\n \to \rone^\n$ is given by 
$$
\left(\lp \dern_\jj  \dern_\jjj^*\rp v\right)_\ii= 
\scalh{(\derh_\jj\kk)_{x_\ii}}{\dern^*_\jjj v} =
\frac{1}{\n} \sum_{\iii=1}^\n \scalh{(\derh_\jj\kk)_{x_\ii}}{(\derh_\jjj\kk)_{x_\iii}}v_\iii 
 = {(\matL_{\jj,\jjj})}_{\ii,\iii} v_\iii
$$
for $i=1,\dots,\n$, for all $v\in\rone^\n$.
Then, since $\dern_\jj\gradn^* \beta = \sum_{\jj=1}^\dd \dern_\jj {\dern_\jjj}^*(\beta_{\jj,\ii})_{\ii=1}^\n$, 
we have that $\matL_\jj$ is the matrix associated to the operator
$\dern_\jj\gradn^*:\rone^{\n\dd}\to\rone^\n$, that is 
$$
(\dern_\jj\gradn^* \beta)_\ii = \sum_{\iii=1}^\n\sum_{\jjj=1}^\dd (\matL_{\jj,\jjj})_{\ii,\iii}\beta_{\jjj,\iii},
$$
for $\ii=1,\dots,\n$, for all $\beta\in\rone^{\n\dd}$.
To prove equation  \eqref{repr_itext}, first note that, as we have done in Proposition \ref{lemma:repr_theo} extended, \eqref{repr_itext} can be equivalently rewritten as $f^\itext=\Sx^*\alpha^{\itext} +\gradn^*\beta^{\itext}$. We now proceed by induction. The base case, namely the representation for $\itext=0$ and $\itext=1$,  is clear.
   Then, by the inductive hypothesis we have that
$f^{\itext-1} =   \Sx^* \alpha^{\itext-1} +\gradn^*\beta^{\itext-1}$, and 
$f^{\itext-2} =   \Sx^* \alpha^{\itext-2} +\gradn^*\beta^{\itext-2}$
so that 
$ \tilde{f}^{\itext} =   \Sx^* \tilde{\alpha}^{\itext} +\gradn^*\tilde{\beta}^{\itext}$
with $\tilde{\alpha}^\itext$ and $\tilde{\beta}^\itext$ defined by \eqref{update_s} and \eqref{update_tilde}.
   Therefore, using \eqn{basic_prox}, \eqn{itext}, \eqn{eq:Moreau} it follows that $f^{\itext}$ can be expressed as:
$$
   \lp I-\pi_{\frac{\tau}{\sigma}\Cn}\rp \lp      \Sx^* \lp \lp1-\frac{\ptwo\taumu}{\step} \rp \tilde{\alpha}^\itext- \frac{1}{\step} \lp \matK \tilde{\alpha}^\itext+ \matZ \tilde{\beta}^\itext -\vy \rp\rp + \lp1-\frac{\ptwo\taumu}{\step} \rp\gradn^* \tilde{\beta}^\itext
   \rp 
$$
and the proposition is proved, letting $\tilde{\alpha}^{\itext}$, $\tilde{\beta}^{\itext}$ and
$\bar{v}^{\itext}$ as in Equations \eqref{update_a}, \eqref{update_b} and \eqref{vbar}.

For the projection, we first observe that operator $\dern_\jj\Sx^*:\rone^\n\to\rone^\n$ is given by
$$
(\dern_\jj\Sx^* \alpha)_\ii= (\scalh{\Sx^*\alpha}{(\derh_\jj k)_{x_\ii}})=
\frac 1 \n  \sum_{\iii=1}^\n\alpha_\iii (\derh_\jj k)_{x_\ii}(x_\iii) =
 \sum_{\iii=1}^\n\alpha_\iii (\matZ_\jj)_{\iii,\ii}=
\matZ_\jj^T\alpha.
$$
Then, we can plug the representation \eqn{func_rep} in \eqref{itint}
to obtain \eqref{update_v}.  
\end{proof}

\begin{proof}[Proof of Proposition \ref{prop:eulero}]
Since $\fz$ is the unique  minimizer of the functional $\emp^{\pone}$, 
it satisfies the Euler equation for $\emp^{\pone}$
$$
0\in\partial(\emp(\fz) + 2\pone\Regn(\fz) + \pone\taumu\norh{\fz}^2).
$$
where, for an arbitrary $\lambda>0$, the subdifferential of $\lambda\Regn$ at $f$ is given by
\begin{align*}
\partial \lambda \Regn (f) =& \{\gradn^* v , v\!=\!(v_\jj)_{\jj=1}^\dd\!\in\!\!(\rone^\n)^\dd ~ |~ v_\jj = \lambda {\dern_\jj f}/{\nor{\dern_\jj f}_\n}  \text{ if } \nor{\dern_\jj f}_\n\!>\!0, 
\\ &\  \text{ and }
 \nor{v_\jj}_\n \leq \lambda \text{ otherwise}, \forall \jj=1,\dots,\dd \}
\end{align*}

Using the above characterization and the fact that $\emp  + \pone\taumu\norh{\cdot}^2$ is differentiable, the Euler equation is equivalent to 
$$
\gradn^* v = -\frac{1}{2\step}\nabla (\emp  + \pone\taumu\norh{\cdot}^2)(\fz),
$$
for any $\step>0$ , and for some $v = (v_\jj)_{\jj=1}^\dd\in(\rone^\n)^\dd$ with 
$v = (v_\jj)_{\jj=1}^\dd\in(\rone^\n)^\dd$ such that
\begin{eqnarray}
v_{\jj}&=&\frac{\pone}{\step}\frac{\dern_\jj \fz}{\nor{\dern_\jj \fz}_\n} \quad\textrm{if~} \quad \nor{\dern_\jj \fz}_\n>0,\nonumber\\
v_{\jj}&\in&\frac{\pone}{\step}\Bn \qquad\qquad\textrm{~otherwise}.\nonumber
\end{eqnarray}
In order to prove \eqref{eq:selection}, we proceed by contradiction and  assume  that  $\nor{\dern_\jj \fz}_\n>0$.
This would imply $\nor{v_\jj}_\n=\pone/\step$, which contradicts the assumption, hence $\nor{\dern_\jj \fz}_\n=0$.\\
We now prove \eqref{eq:vbarok}.
First, according to Definition \ref{def:inexactprox} (see  also Theorem 4.3 in \cite{salzo2011prox}  and \cite{beck09} for the case when the proximity operator is evaluated  exactly), the algorithm generates by construction sequences $\tilde{f}^\itext$ and $f^\itext$ such that
\[
\tilde{f}^\itext-f^\itext-\frac{1}{2\step}\nabla F(\tilde{f}^\itext) \in \frac{1}{2\step}\partial_{{\step}(\eps^t)^2}2\pone\Regn(f^\itext)=\frac{\tau}{\sigma}\partial_{\frac{\step}{2\pone}(\eps^t)^2}\Regn(f^\itext)  .
\]
where $\partial_\eps$ denotes the $\eps$-subdifferential\footnote{
Recall that the $\eps$-subdifferential, $\partial_\epsilon$,  
of a convex functional $\Omega:\hh\to \rone\cup\{+\infty\}$  is defined as the set
$$
\partial_\eps  \Omega (f) := \{h \in \hh\,:\, \Omega(g)-\Omega(f)\geq \scalh{h}{g-f}-\epsilon,~~ \forall g\in\hh\}, \qquad \forall f\in\hh.
$$
}.
Plugging the definition of $f^\itext$ from \eqref{itext} in the above equation, we obtain  
 $\gradn^*\bar{v}^\itext\in \frac{\pone}{\sigma}\partial_{\frac{\step}{2\pone}(\eps^t)^2}\Regn(f^\itext)$.  
 Now, we can use a kind of   {\em transportation formula} \cite{lemarechal1993} for the $\eps$-subdifferential to find $\tilde{\eps}^\itext$ such that 
 $\gradn^*\bar{v}^\itext\in \frac{\tau}{\sigma} \partial_{\frac{\step}{2\pone}(\tilde{\eps}^t)^2}\Regn(f^\pone)$. By definition of $\eps$-subdifferential:
 \[
 \Regn(f)-\Regn(f^{\itext}) \geq \langle \frac{\step}{\pone}\gradn^*\bar{v}^\itext ,f-f^\itext \rangle_\hh -\frac{\step}{2\pone}(\eps^\itext)^2,\qquad \forall f\in\hh.
 \]
 Adding and subtracting $\Regn(\fz)$ and $\langle \frac{\step}{\pone}\gradn^*v^\itext ,\fz \rangle$ to the previous inequality we obtain
 \[
 \Regn(f)-\Regn(\fz) \geq \langle \frac{\step}{\pone}\gradn^*\bar{v}^\itext ,f-\fz \rangle_\hh - \frac{\step}{2\pone}(\tilde{\eps}^\itext)^2\,, \]
  with 
  \[  (\tilde{\eps}^\itext)^2=(\eps^\itext)^2+\frac{2\pone}{\step}\Big(\Regn(f^\itext)-\Regn(\fz)\Big)+\langle 2\gradn^*\bar{v}^\itext, f^\itext-\fz\rangle_\hh.
 \]
 From the previous equation, using \eqref{lim-f} we have
 \beeq{eq:epstilde}{
 (\tilde{\eps}^\itext)^2=(\eps^\itext)^2+\sqrt{\frac{C}{\taumu\pone}}\lp\frac{\pone}{\sigma}\sum_\jj \sqrt{\nor{\dern_\jj^*\dern}} +1\rp\frac{4}{t},
}
which implies $ \frac{\step}{\pone}\gradn^*\bar{v}^\itext\in \partial_{\step(\tilde{\eps^\itext})^2/2\pone} \Regn(\fz)$.
 Now, relying on the  structure of $\Regn$, it is easy to see that 
 \[
 \partial_\eps \Regn (f) \subseteq \{\gradn^* v , v\!=\!(v_\jj)_{\jj=1}^\dd\!\in\!\!(\rone^\n)^\dd ~ |~\nor{v_\jj }_\n\geq  1-{\eps}/\nor{\dern_\jj f}_\n   \text{ if } \nor{\dern_\jj f}_\n\!>\!0\} .
 \]
 Thus, if $\nor{\dern_\jj \fz}_\n>0$ we have $\nor{\bar{v}^\itext}_n\geq\frac{\pone}{\step}(1-\frac{(\tilde{\eps^\itext})^2}{2\nor{\dern_\jj \fz}_\n})$.
  \end{proof}
\section{Proofs of Section \ref{sec:cons}} \label{app:cons}

We start proving the  following preliminary  probabilistic inequalities.
\begin{lemma}\label{lemma:prob}
For $0<\pr_1, \pr_2, \pr_3, \pr_4\le 1$, $\n\in\nat$, it holds\\[2ex]
\begin{tabular}{@{(}l@{)~~}c@{~$$\quad}l@{~~ with ~}l}
1&		&$\Prob{\left|\nor{\vy}_\n^2  - \int_{\X\times\Y}y^2 d\prob(x,y)\right|\le \epsilon(\n,\pr_1)}\! \ge\! 1\!-\!\pr_1 $ 
	 &$ \epsilon(\n,\pr_1)\! = \! \dfrac{2\sqrt{2}}{\sqrt{\n}} M^2 \log{\dfrac{2}{\pr_1}}$,\\[2ex]
2&&$\Prob{\norh{ \Sx^*\vy-\Ik^*\re } \le \epsilon(\n,\pr_2)}\! \ge\! 1\!-\!\pr_2$ 
	&$\epsilon(\n,\pr_2)\!= \!\dfrac{2\sqrt{2}}{\sqrt{\n}} \ka M\log{\dfrac{2}{\pr_2}}$,\\[2ex]
3&		&$\Prob{\nor{\Sx^*\Sx -\Ik^*\Ik}\le \epsilon(\n,\pr_3)} \!\ge\! 1\!-\!\pr_3$
	&$ \epsilon(\n,\pr_3) \!= \! \dfrac{2\sqrt{2}}{\sqrt{\n}}   \ka^2 \log{\dfrac{2}{\pr_3}}$,\\[2ex]
4&		&$\Prob{\nor{{\dern_\jj}^*\dern_\jj -\der_\jj^*\der_\jj}\le \epsilon(\n,\pr_4)}\! \ge\! 1\!-\!\pr_4 $ 
		&$ \epsilon(\n,\pr_4)\! = \! \dfrac{2\sqrt{2}}{\sqrt{\n}}   \kader^2 \log{\dfrac{2}{\pr_4}}.$
\end{tabular}
\end{lemma}

\begin{proof}
From standard concentration inequalities for Hilbert space valued random variables -- see for example \cite{pinsak85}-- we have that,  
if $\xi$ is a random variable with values in a Hilbert space $\hh$ bounded by $L$ and $\xi_1, \dots, \xi_n$ are $n$ i.i.d. samples, then  
$$
  \nor{\frac 1\n \sum_{\ii=1}^\n\xi_\ii-\mathbb{E}(\xi)}\le \epsilon(\n,\pr) =  \frac{2\sqrt{2}}{\sqrt{\n}}   L \log{\frac{2}{\pr}}
$$
with probability at least $1-\eta$, $\eta\in [0,1]$.
The proof is a direct application of the above inequalities to the random variables, \\
\begin{tabular}{l@{~~}@{$\xi=$}l@{~~~~~~~$\xi\in$~~~}l@{~~~~with~~~}l}
{\itshape (1)}&$y^2$&$\rone$			&$\sup_\vz \nor{\xi} \leq M^2$,\\
{\itshape (2)}&$k_x y$&$\hh\otimes\rone$		&$\sup_\vz \nor{\xi} \leq\ka M$,\\
{\itshape (3)}&$\scalh{\cdot}{k_x} k_x$&${\mathcal HS}(\hh)$	&$\sup_\vz \nor{\xi}_{\mathcal HS}(\hh) \leq \ka^2$,\\
{\itshape (4)}&$\scalh{\cdot}{(\derh_\jj k)_x}(\derh_\jj k)_x$&${\mathcal HS}(\hh)$ 	&$\sup_\vz \nor{\xi}_{\mathcal HS}(\hh) \leq \kader^2$.
\end{tabular}\\
where ${\mathcal HS}(\hh), \nor{\cdot}_{{\mathcal HS}(\hh)}$ are 
the space of Hilbert-Schmidt operators on $\hh$ and  the corresponding norm,  respectively 
(note that in the final bound we upper-bound the   operator norm by the Hilbert-Schmidt norm).

\end{proof}

\paragraph{Proofs of the Consistency of the Regularizer.}
We restate Theorem \ref{teo:reg} in an extended form.
\ \\
{\bf Theorem [Theorem \ref{teo:reg} Extended]}
{\it Let $r<\infty$, then under assumption (A2), for any $\eta>0$,
\beeq{eq:reg}{
\Prob{
\sup_{\norh{f}\leq r}|\Regn(f)-\Reg(f)| \geq
r \dd  
\frac{2\sqrt{2}}{(\n)^{1/4}} \kader \sqrt{\log{\frac{2\dd}{\pr}}}
}<\eta.}
Consequently
$$
 \lim_{\n\to\infty}\Prob{\sup_{\norh{f}\leq r}|\Regn(f)- \Reg(f)|> \epsilon}=0, \qquad \forall \epsilon>0.
 $$}

\begin{proof}

For  $f\in \hh$ consider the following chain of inequalities,
\begin{align*}
|\Regn(f)-\Reg(f)|
&\leq \sum_{\jj=1}^\dd \left|\nor{\dern_\jj f}_\n - \nor{\der_\jj f}_{\marg} \right| \\
&\leq \sum_{\jj=1}^\dd \lp \left|\nor{\dern_\jj f}_\n^2 -\nor{\der_\jj f}_{\marg}^2\right|\rp^{1/2} \\ 
 &=\sum_{\jj=1}^\dd \lp\left|\scalh{f}{({\dern_\jj}^*\dern_\jj - {\der_\jj}^*\der_\jj )f}\right|\rp^{1/2} \\
 &\leq\sum_{\jj=1}^\dd \nor{{\dern_\jj}^*\dern_\jj - {\der_\jj}^*\der_\jj}^{1/2}\norh{f},
\end{align*}
that follows from from  $|\sqrt{x}-\sqrt{y}|\leq\sqrt{|x-y|}$, the definition of $\dern_\jj, \der_\jj$ and basic inequalities. 
 Then, using $\dd$ times inequality  (d) in  Lemma \ref{lemma:prob}  with $\pr/\dd$ in place of $\pr_4$, 
and taking the supremum on $f\in\hh$ such that $\norh{f}\leq r$, 
we have  with probability $1-\pr$,
$$
\sup_{\norh{f}\leq r}|\Regn(f)-\Reg(f)| \leq 
r \dd  
\frac{2\sqrt{2}}{(\n)^{1/4}} \kader \sqrt{\log{\frac{2\dd}{\pr}}}.
 $$
The  last statement of the theorem follows easily.
  \end{proof}
  
  \paragraph{Consistency Proofs.}
  
 To prove Theorem \ref{teo:consistency}, we  need the following lemma.
\begin{lemma}\label{lemma:e}
Let $\eta\in(0,1]$. Under  assumptions A1 and A3, we have 
$$
\sup_{\norh{f}\leq r}|\emp(f)-\err(f)|\leq
  \frac{2\sqrt{2}}{\sqrt{\n}}   \left(\ka^2r^2 +2\ka M r+M^2\right) \log{\frac{6}{\pr}},
$$
with probabilty $1-\pr$.
\end{lemma}
\begin{proof}
Recalling the definition of  $\Ik$ we have that, 
\begin{align*}
\err(f) 
 &= \int_{\X\times\Y}(\Ik f(x)-y)^2 d\prob(x,y)  \\
 &= \int_{\X}(\Ik f(x))^2 d\marg(x) + \int_{\X\times\Y}y^2 d\prob(x,y)- 2\int_{\X\times\Y}\Ik f(x)y d\prob(x,y) \\
 &=  \int_{\X}(\Ik f(x))^2 d\marg(x) + \int_{\X\times\Y}y^2 d\prob(x,y)- 2\int_\X \Ik f(x)\re(x) d\marg(x)\\
 &=\scalh{f}{\Ik^*\Ik f} + \int_{\X\times\Y}y^2 d\prob(x,y) -2 \scalh{ f}{\Ik^*\re}.
\end{align*}
Similarly $\emp(f)= \scalh{f}{\Sx^*\Sx f} +\nor{\vy}_n^2 -2 \scalh{ f}{\Sx^*\re}.$ Then, for all $f\in \hh$, we have the bound 
$$
|\emp(f)-\err(f)|\leq \nor{\Sx^*\Sx -\Ik^*\Ik }\nor{f}^2_\hh+2 \nor{\Sx^* \vy-\Ik^*\re}_\hh  \nor{f}_\hh+\left|\nor{\vy}_\n^2  - \int_{\X\times\Y}y^2 d\prob(x,y)\right|
$$
The proof follows  applying Lemma \ref{lemma:prob} with probabilities $\pr_1=\pr_2=\pr_3=\pr/3$.
\end{proof}

We  now prove Theorem \ref{teo:consistency}. We use the  following standard result in regularization theory (see for example \cite{donzol93}) to control the  the approximation error.
\begin{proposition}\label{prop:zolezzi}
Let $\pone_\n\to0$, be a positive sequence. Then  we have that 
$$
\err^{\pone_\n} (f^{\tau_\n}) - \inf_{f\in \hh} \err(f) \to 0.
$$
\end{proposition}

\begin{proof}[Proof of Theorem \ref{teo:consistency}]
We recall the   standard  sample/approximation  error decomposition 
\beeq{decomp}{\err(\fz)-\inf_{f \in \hh} \err(f)\le |\err(\fz) - \err^{\pone}(f^\tau)|+
|\err^{\pone}(f^\tau) - \inf_{f \in \hh} \err(f)|
}
where  $\err^{\pone}(f)=\err(f)+2\pone\Reg(f)+\ptwo \taumu\nor{f}_\hh^2.$

We first consider the sample error. Toward this end, we note that 
$$
\pone\taumu\norh{\fz}^2\leq\emp^{\pone}(\fz)\leq \emp^{\pone}(0)=\nor{\vy}_\n^2\Longrightarrow
\nor{\fz}_\hh\leq \frac{\nor{\vy}_\n}{\sqrt{\ptwo \taumu}}\leq \frac{M}{\sqrt{\ptwo \taumu}},$$ 
and similarly $\nor{f^\tau}_\hh\leq  (\int_\X y^2 d\rho)^{1/2}/\sqrt{\ptwo \taumu}\leq \frac{M}{\sqrt{\ptwo \taumu}}$.\\
We have the following bound,
\begin{align*}
\err(\fz)-\err^{\pone}(f^\tau)
&\leq (\err(\fz) - \emp(\fz))+\emp(\fz)-\err^{\pone}(f^\tau)\\
&\leq (\err(\fz) - \emp(\fz))+\emp^{\pone}(\fz)-\err^{\pone}(f^\tau)\\
&\leq (\err(\fz) - \emp(\fz))+\emp^{\pone}(f^\tau)-\err^{\pone}(f^\tau)\\
&\leq (\err(\fz) - \emp(\fz))+(\emp(f^\tau) - \err(f^\tau))+
\pone(\Regn(f^\tau) - \Reg(f^\tau))\\
&\leq 2\sup_{\nor{f}_\hh\leq \frac{M}{\sqrt{\ptwo \taumu}}} 
|\emp(f)-\err(f)|+\pone\sup_{\nor{f}_\hh\leq \frac{M}{\sqrt{\ptwo \taumu}}} |\Regn(f)-\Reg(f)|.
\end{align*}
Let $\pr'\in (0,1]$. Using Lemma \ref{lemma:e} with probability $\pr = 3\pr'/(3+\dd)$, and 
inequality \eqref{eq:reg} with  $\pr = \dd\pr'/(3+\dd)$, and if $\pr'$ is sufficiently small we obtain 
$$
\err(\fz)-\err^{\pone}(f^\tau) \leq
  \frac{4\sqrt{2}}{\sqrt{\n}}   M^2 \left(\frac{\ka^2}{\ptwo \taumu} + \frac{2\ka}{\sqrt{\ptwo \taumu}} +1\right) \log{\frac{6+2\dd}{\pr'}}
+ \tau  \frac{2\sqrt{2}}{(\n)^{1/4}}   \dd  
\frac{M}{\sqrt{\ptwo \taumu}} \kader \sqrt{\log{\frac{6+2\dd}{\pr'}}} .
 $$
 with probability $1-\pr'$. 
 Furthermore,  we have the bound 
 \beeq{sample_err}{
\err(\fz)-\err^{\pone}(f^\tau) \leq
c \left( \frac{M\ka^2}{n^{1/2} \pone \nu}+\frac{\tau^{1/2} \dd\kader}{n^{1/4}\sqrt{\nu}} \right)  \log{\frac{6+2\dd}{\pr'}}
}
where $c$ does not depend on $\n,\pone,\taumu,\dd$.
The  proof follows, if we plug \eqref{sample_err} in \eqref{decomp} and  take
$\pone=\pone_\n$ such that $\pone_\n\to0$ and $(\pone_\n\sqrt{\n})^{-1}\to0$, 
since   the approximation error goes to zero (using Proposition \ref{prop:zolezzi}) 
and the sample error goes to   zero in probability as   $\n\to\infty$ by \eqref{sample_err}.

\end{proof}

We next consider convergence in the RKHS  norm.
The following result on the convergence of the approximation error is standard  \cite{donzol93}.
\begin{proposition}\label{prop:zolezziB}
Let $\ptwo_\n\to0$, be a positive sequence. Then  we have that 
$$
\norh{\fdag - f^{\pone_\n}}\to 0.
$$
\end{proposition}
\noindent We can now prove Theorem \ref{teo:strong_consistency}.
The main difficulty is to control the sample error in the $\hh$-norm. This requires  showing 
 that controlling the distance between the minima of two functionals,  we can control the distance between  their minimizers. 
 Towards this end it is critical to use 
the results in \cite{villa2010epi} based on Attouch-Wetts convergence.
We need to recall some useful quantities.
Given two subsets $A$ and $B$ in a metric space  $(\hh, d)$, the excess of $A$ on $B$ is defined as
 $e(A,B):=\sup_{f\in A}d(f,B)$, with the convention that $e(\emptyset,B)=0$ for every $B$. 
Localizing the definition of the excess we get the quantity 
$e_r(A,B): e(A\cap B(0,r),B)$ for each ball $B(0,r)$ of radius $r$ centered at the origin.
The $r$-{\em epi-distance} between two subsets $A$ and $B$ of $\hh$,
is denoted by $d_r(A,B)$ and is defined as
$$
d_r(A,B) :=  \max \{e_r(A,B),e_r(B,A)\}.
$$
 The notion of epi-distance can be extended to any two functionals $F,G :\hh\to\rone$
by
$$
d_r(G,F) :=  d_r(\text{epi}(G),\text{epi}(F)),
$$
where 
for any $F:\hh\to\rone$, \text{epi}(F) denotes the epigraph of $F$ defined as
$$
\text{epi}(F):= \{(f,\alpha), F(f)\leq\alpha\}.
$$

We are now ready to prove Theorem \ref{teo:strong_consistency}, which we present here in an extended  form.
\ \\
{\bf Theorem [Theorem \ref{teo:strong_consistency} Extended]}
{\it Under assumptions A1, A2 and A3, 

\beeq{bound_in_h}{
\Prob{\norh{\fz-\fdag} \geq A(n,\pone)^{1/2}  +\norh{f^\pone-\fdag},
}< \eta
}

where
 \[A(n,\pone)=4\sqrt{2}M   \left(\frac{4\ka^2 M}{\sqrt{\n}\pone^2\taumu^2} +\frac{4\ka}{\sqrt{\n}\pone\taumu\sqrt{\pone\taumu}}+\frac{1}{\sqrt{\n}\pone\taumu}
+\frac{2 \dd  \kader}{\n^{1/4}\taumu\sqrt{\pone\taumu}}
 \right)\]

for  $0<\eta \le 1$. Moreover, 
$$
 \lim_{\n\to\infty} \Prob{ \nor{\fzn-\fdag}_\hh\geq \epsilon}=0,\qquad \forall \epsilon>0,
$$
for any $ \pone_\n$ such that $ \pone_\n\to 0$ and $(\sqrt{\n}\pone^2_\n)^{-1}\to 0$.
}

\begin{proof}[Proof of Theorem \ref{teo:strong_consistency}]

We consider the decomposition
of $\norh{\fz-\fdag}$ into a sample and approximation  term,
\beeq{decompB}{
\norh{\fz-\fdag}\leq \norh{\fz-f^{\pone}} +\norh{f^\pone-\fdag}.
}
From Theorem 2.6 in \cite{villa2010epi} we have that
$$
\psi_{\ptwo \taumu }^{\diamond}(\norh{\fz-f^\tau})\leq
4d_{M/\sqrt{\ptwo \taumu}}(t_{\err^{\pone}}\err^{\pone},t_{\err^{\pone}}\emp^{\pone})
$$
where $\psi_{\ptwo \taumu }^{\diamond} (t):=\inf \{\frac{\ptwo \taumu}{2}s^2+|t-s|:s\in[0,+\infty)\}$, 
and $t_{\err^{\pone}}$ is the translation map defined as
$$
t_{\err^{\pone}} G(f) = G(f+f^\tau) - \err^{\pone}(f^\tau)
$$
for all $G:\hh\to\rone$.

From  Theorem 2.7  in \cite{villa2010epi}, we have that
$$
d_{M/\sqrt{\ptwo \taumu}}(t_{\err^{\pone}}\err^{\pone},t_{\err^{\pone}}\emp^{\pone})
\leq \!\!\!\!\sup_{\norh{f}\leq M/\sqrt{\ptwo \taumu}} |t_{\err^{\pone}}\err^{\pone}(f) - t_{\err^{\pone}}\emp^{\pone}(f)|.
$$
We have the bound, 
\mfi
\sup_{\norh{f}\leq M/\sqrt{\ptwo \taumu}} |t_{\err^{\pone}}\err^{\pone}(f) - t_{\err^{\pone}}\emp^{\pone}(f)|
&\leq& \!\!\!\!\sup_{\norh{f}\leq M/\sqrt{\ptwo \taumu}+\norh{f^\pone}} |\err^{\pone}(f) - \emp^{\pone}(f)| \\
&\leq& \!\!\!\!\sup_{\norh{f}\leq 2M/\sqrt{\ptwo \taumu}} |\err(f) - \emp(f)| +
\pone \!\!\!\!\!\sup_{\norh{f}\leq 2M/\sqrt{\ptwo \taumu}} |\Reg(f) - \Regn(f)|.
\mff
Using Theorem \ref{teo:reg} (equation \eqref{eq:reg}) and Lemma \ref{lemma:e}
we obtain with probability $1-\pr'$, if $\pr'$ is small enough,
\begin{eqnarray}\label{epi}
d_{M/\sqrt{\ptwo \taumu}}(t_{\err^{\pone}}\err^{\pone},t_{\err^{\pone}}\emp^{\pone})&\leq&
  \frac{2\sqrt{2}}{\sqrt{\n}}   \left(\ka^2\frac{4 M^2}{\pone\taumu} +4\ka\frac{M^2}{\sqrt{\pone\taumu}}+M^2\right)\log{\frac{6+2\dd}{\pr'}}
+\pone\frac{2M}{\sqrt{\pone\taumu}}
\dd  \frac{2\sqrt{2}}{\n^{1/4}} \kader \sqrt{\log{\frac{6+2\dd}{\pr'}}}\nonumber \\
&\leq&
  2\sqrt{2}M   \left(\frac{4\ka^2 M}{\sqrt{\n}\pone\taumu} +\frac{4\ka M}{\sqrt{\n}\sqrt{\pone\taumu}}+\frac{M}{\sqrt{\n}}
+\pone\frac{2 \dd  \kader }{\n^{1/4}\sqrt{\pone\taumu}}
 \right)\log{\frac{6+2\dd}{\pr'}}.
\end{eqnarray}
From the definition of $\psi_{\ptwo \taumu}^{\diamond}$ it is possible to see that we can write explicitly $(\psi_{\ptwo \taumu}^{\diamond})^{-1}$ as
$$
(\psi_{\ptwo\taumu}^{\diamond})^{-1}(y) = \begin{cases}
\sqrt{\frac{2y}{\ptwo\taumu}} &\textrm{~~if~} y<\frac{1}{2\ptwo\taumu}\\
y + \frac{1}{2\ptwo} &\textrm{~~otherwise}.
\end{cases}
$$
Since $\pone=\pone_\n\to0$ by assumption, for sufficiently large $\n$, the bound in \eqref{epi} is smaller than $1/2\ptwo\taumu$, 
and  we obtain that  with probability $1-\pr'$,
\beeq{sample_errB}{
\norh{\fz-f^\tau} \leq
  \lp4\sqrt{2}M   \left(\frac{4\ka^2 M}{\sqrt{\n}\pone^2\taumu^2} +\frac{4\ka}{\sqrt{\n}\pone\taumu\sqrt{\pone\taumu}}+\frac{1}{\sqrt{\n}\pone\taumu}
+\frac{2 \dd  \kader}{\n^{1/4}\taumu\sqrt{\pone\taumu}}
 \right)\rp^{1/2}\sqrt{\log{\frac{6+2\dd}{\pr'}}}.}

If we now plug  \eqref{sample_errB} in \eqref{decompB}
we obtain the first part of the proof.
The rest of the proof follows by taking the limit $\n\to\infty$, and by observing 
that, if one chooses $\pone = \pone_\n$ such that
$\pone_\n\to0$ and $(\pone_\n^2\sqrt{\n})^{-1}\to0$, 
the assumption of Proposition \ref{prop:zolezziB} is satisfied
and the bound  in \eqref{sample_errB} goes to $0$, so that the limit of the sum of the sample and approximation terms goes to $0$.

\end{proof}

\paragraph{Proofs of the Selection properties.} 

In order to prove our main selection result, we will need the following lemma. 
\begin{lemma}\label{lemma:der_rate}
Under assumptions A1, A2 and A3 and defining $A(n,\tau)$ as in Theorem \ref{teo:strong_consistency} extended, we have, for all $\jj=1,\dots,\dd$  and for all $\epsilon>0$, 
$$
\Prob{\left| \nor{\dern_\jj \fz}^2_\n-\nor{\der_\jj \fdag}^2_{\marg}\right|\geq \epsilon}<(6+2\dd)\mathrm{exp}\!\lp-\frac{\epsilon-b(\pone)}{a(\n,\pone)}\rp,
$$
where $a(n,\tau)=2\max \{\frac{2\sqrt{2}M^2 \kader^2}{\sqrt{\n}\pone\taumu},  
   2\kader^2 A(\n,\pone) \}$
and $\lim_{\pone\to0}b(\pone) = 0$.

\end{lemma}
\begin{proof}
We have the following set of inequalities
\begin{align*}
\left|\nor{\dern_\jj \fz}^2_\n - \nor{\der_\jj \fdag}^2_{\marg}\right|
&=|\scalh{\fz}{\dern^*_\jj\dern_\jj \fz} -
\scalh{\fdag}{\der^*_\jj\der_\jj \fdag}+\\
&~~~~~\scalh{\fz}{\der^*_\jj\der_\jj \fz} -\scalh{\fz}{\der^*_\jj\der_\jj \fz}+\\
&~~~~~\scalh{\fdag}{\der^*_\jj\der_\jj \fz} -\scalh{\fdag}{\der^*_\jj\der_\jj \fz}
|\\
&=\left|
\scalh{\fz}{(\dern^*_\jj\dern_\jj-\der^*_\jj\der_\jj) \fz} + 
\scalh{\fz-\fdag}{\der^*_\jj\der_\jj( \fz-\fdag)}
\right|\\
&\leq \nor{\dern^*_\jj\dern_\jj -\der^*_\jj\der_\jj} \frac{M^2}{\pone\taumu} + \kader^2 \norh{\fz-\fdag}^2\\
&\leq \nor{\dern^*_\jj\dern_\jj -\der^*_\jj\der_\jj} \frac{M^2}{\pone\taumu} + 2\kader^2 \norh{\fz-f^\pone}^2
+ 2\kader^2 \norh{f^\pone-\fdag}^2.
\end{align*}

Using Theorem \ref{teo:strong_consistency} extended, equation \eqref{sample_errB}, and  Lemma \ref{lemma:prob} with probability $\pr_4 = \pr/(3+\dd)$, 
 we obtain with probability $1-\pr$
$$
\left|\nor{\dern_\jj \fz}^2_\n - \nor{\der_\jj \fdag}^2_{\marg}\right|
\leq \frac{2\sqrt{2}M^2 \kader^2}{\sqrt{\n}\pone\taumu} \log{\frac{6+2\dd}{\pr}}+
   2\kader^2 A(\n,\pone) \log{\frac{6+2\dd}{\pr}}+
2\kader^2 \norh{f^\pone-\fdag}^2.
$$
We can further write
$$
\left|\nor{\dern_\jj \fz}^2_\n - \nor{\der_\jj \fdag}^2_{\marg}\right|
\leq  
 a(\n,\pone)  \log{\frac{6+2\dd}{\pr}} +b(\pone),
$$ 
where $a(\n,\pone)=2\max \{\frac{2\sqrt{2}M^2 \kader^2}{\sqrt{\n}\pone\taumu},  
   2\kader^2 A(\n,\pone) \}$
and $\lim_{\pone\to0}b(\pone) = 0$ according to  Proposition \ref{prop:zolezziB}.
The proof follows by writing
$\epsilon = a(\n,\pone)\log{\frac{6+2\dd}{\pr}} +b(\pone)$ and inverting  it with respect to $\pr$.

\end{proof}

\noindent Finally  we can prove Theorem \ref{teo:subset}.
\begin{proof}[Proof of Theorem \ref{teo:subset}]
We have 
$$
\Prob{\I\subseteq \Iznon} = 
1-\Prob{\I\not\subseteq \Iznon} = 
1- \Prob{\bigcup_{\jj\in\I}\{\jj\notin \Iznon\}} \geq
1-\sum_{\jj\in\I} \Prob{\jj\notin \Iznon} 
$$
Let us now estimate $\Prob{\jj\notin \Iznon} $ or equivalently $\Prob{\jj\in \Iznon} =  \Prob{\nor{\dern_\jj \fz}_\n^2>0}$, for $\jj\in\I$. 
Let $C < \min_{\jj\in \I} \|\der_\jj \fdag\|^2_{\marg}$.
From  Lemma \ref{lemma:der_rate},
there exist $a(\n,\pone)$
   and $b(\pone)$ satisfying 
 $\lim_{\pone\to0}b(\pone) = 0$, such that
$$\left|\|\der_\jj \fdag\|_{\marg}^2 - \|\dern_\jj \fz\|_\n^2\right|\leq \epsilon$$ with
 probability $1-(6+2\dd)\mathrm{exp}\!\lp-\frac{\epsilon-b(\pone)}{a(\n,\pone)}\rp$, for all $\jj=1,\dots,\dd$.
Therefore, for $\epsilon=C$, for $\jj\in \I$, it holds
$$
\nor{\dern_\jj \fz}_\n^2\geq \nor{\der_\jj \fdag}_{\marg}^2 - C\gneq0.
$$
with probability $1-(6+2\dd)\mathrm{exp}\!\lp-\frac{C-b(\pone)}{a(\n,\pone)}\rp$. 
We than have
$$
\Prob{\jj\in \Iznon} = \Prob{\nor{\dern_\jj \fz}_\n^2>0}\geq 1-
(6+2\dd)\mathrm{exp} \lp-\frac{C-b(\pone)}{a(\n,\pone)} \rp,
$$ 
so that $\Prob{\jj\notin \Iznon}\leq (6+2\dd)\mathrm{exp} \lp-\frac{C-b(\pone)}{a(\n,\pone)} \rp$.
Finally, if we let $\pone = \pone_\n$ satisfying the assumption, we have 
$\lim_\n b(\pone_\n)\to 0 $,
$\lim_\n a(\n,\pone_\n)\to 0$, 
so that
\begin{eqnarray}
\lim_{\n\to\infty} \Prob{\I\subseteq \Izn} &\geq &
\lim_{\n\to\infty} \left[1-|\I| (6+2\dd)\mathrm{exp} \lp-\frac{C-c(\pone_\n)}{a(\n,\pone_\n)} \rp\right]\nonumber\\
&=& 1-|\I| (6+2\dd) \lim_{\n\to\infty} \mathrm{exp} \lp-\frac{C-b(\pone_\n)}{a(\n,\pone_\n)} \rp \nonumber\\&=& 1.\nonumber
\end{eqnarray}
\end{proof}

\newpage
\begin{table}[h]
\caption{List of symbols and notations}
\label{tab:notations}
\footnotesize{
\begin{center}
{\renewcommand{\arraystretch}{1.6}
\begin{tabular}{p{3cm}|c}
\hline
\hline
Spaces and distributions& \begin{tabular}{p{3cm}|p{9cm}} $\mathcal{X}\subseteq \rone^\dd$ & input space\\
\hline    $\mathcal{Y}\subseteq \rone$ & output space\\ 
\hline $ \rho$ & probability distribution on $\mathcal{X}\times\mathcal{Y}$\\
\hline $ \rho_{\mathcal{X}}$ & marginal distribution of $\rho$\\
\hline    $L^2(\mathcal{X},\rho_{\mathcal{X}})$ & $\{f:\mathcal{X}\to \mathbb{R}\,:\text{measurable and s.t.} \int_{\mathcal{X}} f(x)^2\,d\rho_\mathcal{X}(x)<+\infty\}$\\
\hline    $\hh$ & RKHS $\subseteq \{f:\mathcal{X}\to\mathcal{Y}\}$ \end{tabular} \\ 
\hline\hline
Norms and scalar  products & \begin{tabular}{p{3cm}|p{9cm}} $\nor{\cdot}_n$ and $\langle\cdot,\cdot\rangle_n$ & $\frac{1}{\sqrt{n}} \cdot$ euclidean norm and scalar product\\
\hline $\nor{\cdot}_{\rho_{\mathcal{X}}}$ and $\scal{\cdot}{\cdot}_{\rho_\mathcal{X}}$ & norm and scalar product in $L^2(\mathcal{X},\rho_{\mathcal{X}})$\\
\hline $\nor{\cdot}_\hh$ and $\scalh{\cdot}{\cdot}$ & norm and scalar product in $\hh$ \end{tabular} \\ 
\hline\hline
Functionals and Operators & \begin{tabular}{p{3cm}|p{9cm}} 
$\Reg:\hh\to[0,+\infty)$ &$\Reg(f)=\sum\limits_{\jj=1}^\dd \sqrt{\int_{\mathcal{X}} \left(\frac{\partial f(x)}{\partial x^{\jj}}\,d\rho_{\mathcal{X}}(x)\right)^2}$\\
\hline
$\Regn:\hh\to[0,+\infty)$ &$\Regn(f)=\sum\limits_{\jj=1}^\dd \sqrt{\frac1n\sum\limits_{i=1}^n\left(\frac{\partial f(x_i)}{\partial x^{\jj}}\right)^2}$\\
\hline
$\err:\hh\to[0,+\infty)$& $\err(f)=\int_{\mathcal{X}} (f(x)-y)^2\,d\rho(x,y)$\\
\hline
$\err^\ptwo:\hh\to[0,+\infty)$& $\err^\ptwo(f)=\int_{\mathcal{X}} (f(x)-y)^2\,d\rho(x,y)+\tau(2\Reg(f)+\taumu \nor{f}^2_\hh)$\\
\hline
$\hat{\err}:\hh\to[0,+\infty)$&$ \hat{\err}(f)=\sum\limits_{i=1}^n \frac{1}{n} (f(x_i)-y_i)^2$\\
\hline
$\hat{\err}^\ptwo:\hh\to[0,+\infty)$&$\hat\err^\ptwo(f)=\sum\limits_{i=1}^n \frac{1}{n}(f(x_i)-y_i)^2+\tau(2\Regn(f)+\taumu\norh{f}^2)$\\
\hline
$\Ik:\hh\to\ldue$& $(\Ik f)(x)=\scalh{f}{\kk_x}$\\
\hline
$\hat{S}:\hh\to\mathbb{R}^n$& $\hat{S}f=(f(x_1),\ldots,f(x_n))$\\
\hline
$\der_\jj:\hh\to \ldue$& $(\der_\jj f)(x)=\scal{f}{(\derh_\jj k)_{x}}$\\
\hline
$\dern_\jj:\hh\to\rone^n$& $\dern_\jj(f)=\left(\frac{\partial f}{\partial x^\jj}(x_1),\ldots,\frac{\partial f}{\partial x^\jj}(x_n)\right) $\\
\hline
$\grad\!:\!\hh\!\to \!(\ldue)^\dd$ & $\grad f = (\der_\jj f)_{\jj=1}^\dd$\\
\hline
$\gradn:\hh\to (\rone^\n)^\dd$ &$\gradn f= (\dern_\jj f)_{\jj=1}^\dd$\\
\end{tabular}\\
\hline\hline
Functions & \begin{tabular}{p{3cm}|p{9cm}} 
$k_x: \mathcal{X}\to\rone$ & $t\mapsto k(x,t)$\\
\hline
$f^\dag_\rho$ & $\argmin_{f\in {\argmin \err}} \{\Omega_1^D(f)+\nu\nor{f}_\hh^2\}$\\
\hline
$(\partial_\jj k)_x: \mathcal{X}\to\rone$ & $\left. t\mapsto \frac{\partial k(s,t)}{\partial s^\jj}\right|_{s=x}$\\
\hline
${f}^\ptwo$ & the minimizer in $\hh$ of ${\err}^\ptwo$\\
\hline
$\hat{f}^\ptwo$ & the minimizer in $\hh$ of $\hat{\err}^\ptwo$ \end{tabular} \\ 
\hline\hline
Sets & \begin{tabular}{p{3cm}|p{9cm}} 
$R_\rho$ & $\{\jj\in\{1,\ldots d\}\,:\, \frac{\partial f^\dag_\rho}{\partial x^\jj}\neq 0\}$\\
\hline
$\hat{R}^\tau$ & $\{\jj\in\{1,\ldots d\}\,:\, \frac{\partial f^{\tau}}{\partial x^\jj}\neq 0\}$\\
\hline
$B_n$& $ \{v\in\rone^n\,:\, \nor{v}_n\leq 1\}$\\
\hline
$B_n^d$& $ \{v\in\rone^n\,:\, \nor{v}_n\leq 1\}^d$
\end{tabular}\\
\hline\hline
\end{tabular}}
\end{center}}
\end{table}

\end{document}